\documentclass{article}

\usepackage{microtype}
\usepackage{graphicx}
\usepackage{subcaption}
\usepackage{booktabs}
\usepackage{float}

\usepackage{xcolor}
\usepackage[
  colorlinks = true,
  linkcolor = blue,
  urlcolor  = blue,
  citecolor = blue,
  anchorcolor = blue,
]{hyperref}
\usepackage[most]{tcolorbox}

\usepackage[preprint]{neurips_2025}

\usepackage{amsmath}
\usepackage{amssymb}
\usepackage{mathtools}
\usepackage{amsthm}
\usepackage{interval}
\usepackage{dsfont}
\usepackage{wrapfig}
\usepackage[capitalize,noabbrev]{cleveref}

\theoremstyle{plain}
\newtheorem{theorem}{Theorem}[section]
\newtheorem{proposition}[theorem]{Proposition}
\newtheorem{lemma}[theorem]{Lemma}

\theoremstyle{definition}
\newtheorem{definition}[theorem]{Definition}
\newtheorem{assumption}[theorem]{Assumption}
\theoremstyle{remark}

\usepackage{mdframed}
\usepackage{xpatch}
\makeatletter
\xpatchcmd{\endmdframed}
  {\aftergroup\endmdf@trivlist\color@endgroup}
  {\endmdf@trivlist\color@endgroup\@doendpe}
  {}{}
\makeatother

\mdfdefinestyle{theoremstyle}{%
linecolor=black,linewidth=0.8pt,%
topline=false,bottomline=false,rightline=false,%
leftmargin =-.0cm, rightmargin=+.0cm,
usetwoside=false,
skipabove=5pt,
skipbelow=0pt,
innertopmargin=5pt,
innerbottommargin=2pt,
innerrightmargin=0pt,
innerleftmargin=5pt,
}
\theoremstyle{plain}

\newcommand{\TextProp}[1]{{\color{red}P??}}

\newcommand{\Sec}[1]{Sec.~\ref{sec:#1}}

\newcommand{\Appendix}[1]{Appendix~\ref{appendix:#1}}
\newcommand{\Figure}[1]{Figure~\ref{fig:#1}}
\newcommand{\Fig}[1]{Fig.~\ref{fig:#1}}
\newcommand{\Proposition}[1]{Proposition~\ref{prop:#1}}

\newcommand{\Lemma}[1]{Lemma~\ref{lem:#1}}
\newcommand{\Assumption}[1]{Assumption~\ref{assumption:#1}}

\usepackage{svg}
\usepackage{pgfplots}
\pgfplotsset{compat=1.18}
\usepgfplotslibrary{statistics}
\usetikzlibrary{pgfplots.statistics}
\usetikzlibrary{positioning, arrows.meta, shapes.geometric, calc}
\usepgfplotslibrary{fillbetween}
\usetikzlibrary{fit}

\pgfplotsset{
    discard if not/.style 2 args={
        x filter/.append code={
            \edef\tempa{\thisrow{#1}}
            \edef\tempb{#2}
            \ifx\tempa\tempb
            \else
                
            \fi
        }
    }
}

\pgfdeclareplotmark{mystar}{
    \node[star,star point ratio=2.25,minimum size=6pt,
          inner sep=0pt, solid,
          draw=\markerdraw, fill=\markerfill
          ] {};
}

\DeclareMathOperator*{\argmin}{arg\,min}

\newcommand{\pnorm}[2]{\left\lVert #2 \right\rVert_{#1}}
\newcommand{\pprod}[3]{\left\langle #2, #3 \right\rangle_{#1}}

\usepackage{xr}
\usepackage{xcolor}
\usepackage{ifthen}

\newboolean{showcomments}
\setboolean{showcomments}{true}
\ifthenelse{\boolean{showcomments}}
{ \newcommand{\mynote}[3]{
		\fbox{\bfseries\sffamily\scriptsize#1}
		{\small$\blacktriangleright$\textsf{\emph{\color{#3}{#2}}}$\blacktriangleleft$}}
	\newcommand{\zzz}[1]{{\setlength{\fboxsep}{2pt}\fcolorbox{black}{yellow}{\textsf{\emph{#1}}}}\xspace}}
{ \newcommand{\mynote}[3]{}
	\newcommand{\zzz}[1]{}}

\definecolor{tab1}{RGB}{31, 119, 180}
\definecolor{tab2}{RGB}{255, 127, 14}
\definecolor{tab3}{RGB}{44, 160, 44}
\definecolor{tab4}{RGB}{214, 39, 40}
\definecolor{tab5}{RGB}{148, 103, 189}

\usepackage[textsize=tiny]{todonotes}
\usepackage{thmtools, thm-restate}

\author{%
  David~A.~R.~Robin
  \\
  Université Paris Dauphine\\
  PSL Research University \\
  \And
  Rafaël Pinot \\
  Sorbonne Université \\
  LPSM, UMR 8001 \\
  \And
  Yann Chevaleyre \\
  Université Paris Dauphine \\
  PSL Research University \\
}

\begin{document}

\title{ Robustness Cannot be Reduced to Regularization: \\
    Studying Adversarial Training Beyond the Linear Case }

\maketitle{}

\begin{abstract}
    The vulnerability of ML models to adversarial examples has recently emerged as a major concern. While adversarial training is one of the most effective countermeasures to this issue, its high computational cost remains an obstacle to practical deployment. Recent progress in reducing this cost has relied, in the case of linear models, on a formal equivalence between the adversarial risk and a simpler form of regularized risk.
    This enabled significantly more efficient training procedures, which naturally raises the question of whether such an equivalence can be extended beyond linear models. In this work, we formally show that \emph{no such equivalence is possible} for two-layer networks. Our proofs proceed via a reduction to key properties that fundamentally separate the adversarial risk from any simple regularized risk which would only exhibit a weak form of data dependence. Beyond this setting, we provide empirical evidence on Wide-ResNets indicating that the same type of impossibility persists in deeper and more expressive architectures.
\end{abstract}

\section{Introduction}
\label{sec:intro}

Adversarial example attacks recently became a major concern in the machine learning community. An adversarial attack refers to a small, imperceptible change of an input that is maliciously designed to fool a machine learning algorithm~\citep{biggio2013evasion,Szegedy2013IntriguingPO,goodfellow2015explaining}. The vulnerability of state-of-the-art classifiers to such attacks raises significant security concerns, particularly for deep neural networks deployed in safety-critical applications~\citep{kumar2019failure,paleyes2022challenges,carlini2024aligned}. This limitation of existing models can largely be attributed to the standard training paradigm, which focuses on minimizing a point-wise loss
${\mathcal{L} : \Theta \times \mathcal{X} \times \mathcal{Y} \to \mathbb{R}}$
averaged along a ground-truth distribution $\mathcal{D}$ over $\mathcal{X} \times \mathcal{Y}$, i.e., \vspace{3pt}
\[
      \mathcal{R}(\mathcal{D}; \theta) \> := \> \mathbb{E}_{(x,y) \sim \mathcal{D}} \left[ \mathcal{L}(\theta, x, y) \right] \quad  \forall \theta \in \Theta.
\]

However, as noted early on by \citet{goodfellow2015explaining}, this objective does not explicitly encourage the model to exhibit regularity properties beyond minimizing this expected loss (a.k.a. risk). Consequently, even when the point-wise loss of a model $\theta$ is small on a sample $(x,y)\sim\mathcal{D}$, an adversary can construct a perturbed input $x+\tau$ (an adversarial example) such that $\lVert \tau \rVert$ is small under a reference norm\footnote{%
Much theoretical progress has been achieved for attacks with bounded $\ell_2$
norm \citep{li2019inductive}, but the main target in image classification
remains $\ell_\infty$-norm attacks, which are not perceived by humans.
} $\lVert \cdot \rVert$, but for which the loss increases substantially.
To account for adversarial manipulations of the inputs, the standard objective must be revised to explicitly incorporate the adversary. The goal becomes minimizing the worst-case expected loss (a.k.a. adversarial risk) under $\delta$-bounded perturbations.
This adversarial risk  $\mathcal{R}_\delta$ can be written as
\[
  \mathcal{R}_\delta(\mathcal{D}; \theta) \> := \>
  \mathbb{E}_{(x,y) \sim \mathcal{D}}
  \left[
    \sup_{\lVert \tau \rVert \leq \delta}
    \mathcal{L}(\theta, x + \tau, y)
    \right] \quad \forall \theta \in \Theta.
\]

The main challenge in minimizing adversarial risk lies in the intractability of the associated min-max objective. Nevertheless, a substantial body of work has focused on solving this problem approximately through surrogate optimization schemes. These approaches, commonly referred to as adversarial training, currently constitute the state-of-the-art defense against adversarial attacks~\citep{croce2020reliable,bartoldson2024adversarial}. In essence, adversarial training involves generating, at each optimization step, an adversarial example for every data point, typically via gradient ascent~\citep{madry2018towards,zhang2019theoretically}. Despite its strong empirical performance, this procedure is computationally demanding, as perturbations must be recomputed at every iteration. Moreover, it remains unclear how closely the resulting solutions approximate the true adversarial optimum.
Consequently, understanding the impact of adversarial training on the learning dynamics of a complex model remains an open problem.

\textbf{Actionable rewriting in the linear case.} To address these limitations, recent works focused on reformulating the adversarial risk in more interpretable/tractable manners, yielding identical robustness at lower computational cost,
by expressing the adversarial risk $\mathcal{R}_\delta$ as a regularized objective of the form \begin{align}
\mathcal{R}_\delta(\mathcal{D};\theta) = \mathcal{R}(\mathcal{D};\theta) + \lambda \cdot \Omega_\delta(\theta)  \quad \forall \theta \in \Theta,
\label{eq:regularized_adv_problem}
\end{align}

where $\Omega_\delta(\theta)$ is an easily computable and interpretable data-independent regularizer, and $\lambda > 0$ is a positive hyperparameter. Such reformulations can both clarify the mechanisms behind adversarial training and suggest more efficient surrogate methods for enforcing robustness. In this direction, \citet{roth2020adversarial} showed that adversarial training is equivalent to a form of data-dependent operator norm regularization. While this perspective is valuable for interpretation, the data-dependency of the result does not directly yield a simpler or more efficient surrogate in terms of computational cost. More recently~\citet{ribeiro2023regularization} showed that adversarial training for linear models is equivalent to a parameter-shrinking regularization that is inexpensive to compute. In particular, when $\mathcal{L}$ is a margin loss of the form $\mathcal{L}(x,y,\theta) = L(y\langle x,\theta\rangle)$ where $L : \mathbb{R} \to \mathbb{R}$ is a lower-semicontinuous convex non-increasing function, the adversarial risk $\mathcal{R}_\delta$ can be written as\vspace{3pt}
\[
\mathcal{R}_\delta(\mathcal{D};\theta) = \mathbb{E}_{(x,y)\sim\mathcal{D}}\bigl[L\!\left(y\langle x,\theta\rangle - \delta\|\theta\|_*\right)\bigr],
\quad \forall \theta\in\Theta,
\]
where $\|\theta\|_*$ denotes the dual norm of $\theta$.
Although not of the form~\eqref{eq:regularized_adv_problem},
this rewriting remains a practically actionable simplification of the adversarial risk and can be leveraged to design accelerated algorithms~\citep{ribeiro2025efficient}. This result is particularly promising, as its extension to more complex architectures could enable new interpretations and implementations of adversarial training, potentially unlocking robust learning at scale. It is therefore crucial to investigate whether similar reformulations can be derived for more expressive models.

\textbf{What happens beyond the linear case?} In this paper, we take a step towards addressing this open problem by showing that \emph{no such reformulation a priori exists beyond the linear case}. We focus on the simplest class of functions that is strictly more expressive than linear models, namely two-layer networks. Indeed, if no tractable reformulation of the adversarial risk exists even for this class, it is unlikely, a fortiori, that such a rewriting would be possible for more complex models. Formally, we consider the mappings ${f : \Theta \times \mathbb{R}^d \to \mathbb{R}^K}$ parameterized over ${\Theta = \mathbb{R}^{m \times d} \times \mathbb{R}^{m \times K}}$ and defined as \vspace{3pt}
\[  f : ( (w,a), x) \mapsto \sum_{i \in [m]} a_i \, \sigma(w_i \cdot x), \]%, \, \, \forall \theta = (w,a) \in \Theta, \]
where $\sigma$ is a differentiable approximation\footnote{
Formally, we assume $\sigma$ differentiable %on $\mathbb{R}$
and that there exists $\varepsilon > 0$ such that
$\lvert u \rvert > \varepsilon \Rightarrow \sigma(u) = \max(0, u)$.
Taking $\varepsilon$ smaller than the machine epsilon,
this is indistinguishable from ReLU at floating-point precision.\linebreak[3]
This is labeled \Assumption{smooth-relu}, and ensures that $f$ is differentiable at all points, see \Lemma{network-differentiable}.
} of ReLU, which ensures adversarial training is well-defined.
We also consider $\mathcal{L}$ to be the cross-entropy loss $\mathcal{L}(\theta, x,y) = \ell(y, f(\theta, x))$
where
${\ell :
% (y \in [k], u \in \mathbb{R}^k)
(y,z)
\mapsto - z_y + \log( \sum_{ i \in  [K] } \exp(z_i) )}$.
In this context, our contributions can be put as follows.

\textbf{Our Contributions.} We analyze the \emph{adversarial gap}
induced by an arbitrary norm $\lVert \cdot \rVert$, defined as
\[
\mathcal{A}_\delta(\mathcal{D}; \theta)
:= \mathcal{R}_\delta(\mathcal{D}; \theta) - \mathcal{R}(\mathcal{D}; \theta),
\]
and systematically investigate whether it can be expressed as either as a data-independent regularization or a weakly data-dependent regularization  term. Specifically, we consider three families of functions whose structure would lead to curbed computational cost, and show that the adversarial gap belongs to neither of the three.
Firstly, we consider separable functions, such as $\theta \mapsto \sum_i \lvert \theta_i \rvert^p$, which can be rejected by simple arguments derived from studying the shape of
the set of minimizers (\Sec{minima-displacement}).
Secondly, by a tighter analysis, we show that no data-independent function
can be a generic rewriting of the adversarial gap, by studying the rotation
of adversarial gap gradients when data distributions change, which additive regularizations
cannot match (\Sec{rejecting-data-independence}).
Thirdly, we consider a class of loosely data-dependent regularizers studied in
\citet{blanchet2019robust} and \citet{ribeiro2023regularization},
and show again that even for two-layer networks, no such function can be a rewriting of
the adversarial gap, by studying the structural dependence of
adversarial gap gradients on the data distribution
(\Sec{rejecting-loss-modulation}).
The main tools for these analyses are the study of minima on specific slices of
parameter space, and the computation of cosine similarities between gradients of
various regularizers.

\textbf{Notations.}
For $n \in  \mathbb{N}$, we write $[n] := \{ 0, \ldots, n-1 \}$ the set of integers strictly smaller than $n$. For $v \in \mathbb{R}^n$, we write $v = (v_i)_{i \in [n]} = (v_0, \ldots, v_{n-1})$,
and for any $u \in \mathbb{R}^q$, we write
${(u \otimes v) \in \mathbb{R}^{q \times n}}$ their tensor product.
We indicate partial derivatives indifferently as index or with a fraction,
so the derivative of the cross-entropy w.r.t. its second argument
is written ${\partial_z \ell(y,z) = (\partial \ell / \partial z)(y,z)}$.
\linebreak[1]
We abuse the notation $\mathcal{D} = \{ (x,y) \}$ to define a
Dirac distribution supported on a single point $(x,y)$.
For any symmetric positive semi definite (PSD) matrix $Q \in \mathbb{R}^{n \times n}$,
and $(u,v) \in \mathbb{R}^n \times \mathbb{R}^n$, we write $\langle u, v \rangle_Q = u^T \cdot Q \cdot v$ and the associated semi-norm $\lVert u \rVert_Q^2 = \langle u, u \rangle_Q$.

\section{Preliminaries}

To define the set of possible regularizations we consider the set $\mathcal{F} \subsetneq (\Theta \to \mathbb{R})$ of
continuous functions almost-everywhere differentiable
(i.e.\ the set of points at which it is not differentiable
has Lebesgue measure zero). To determine whether the adversarial gap admits a reformulation as a simplified regularizer, we introduce the equivalence set $P \subseteq \mathcal{F}^2 \times \mathbb{R}_+^*$ defined as
\begin{align}
\label{eq:propdefinition}
P = \{\, (\Omega_1, \Omega_2, \lambda) \in \mathcal{F}^2 \times \mathbb{R}_+^* \mid
  \exists c \in \mathbb{R}, \> \Omega_2 = \lambda \cdot \Omega_1 + c \,\}
\end{align}

The set $P$ induces an equivalence relation ${(\Omega_1 \sim_P \Omega_2) \Leftrightarrow (\exists \lambda \in \mathbb{R}_+^*, \, (\Omega_1, \Omega_2, \lambda) \in P)}$.
Accordingly, if we find a regularizer $\Omega_\delta$ such that $\Omega_\delta \sim_P \mathcal{A}_\delta(\mathcal{D}; \cdot)$, then up to a tunable parameter $\lambda \in \mathbb{R}_+^*$, the losses
$\mathcal{R}(\theta) + \mathcal{A}_\delta(\mathcal{D}; \theta) $
and $ {\mathcal{R}(\theta) + \lambda \cdot \Omega_\delta(\theta) }$
are interchangeable during training.
The objective of our study is to determine whether there exists a mapping that is equivalent to $\mathcal{A}_\delta$ while being data-independent, or weakly data-dependent as presented in~\citep{ribeiro2023regularization}.

\textbf{Reduction to limits.}
\footnote{
All limits considered use the pointwise topology, continuity of gradients with respect to the input is sufficient to ensure that $A(\mathcal{D}; \theta)$ is in $\mathcal{F}$ for compactly-supported distributions, which is enforced by differentiability of $\sigma$.
}
A key property of the equivalence relation we just discussed is that it is stable by passage to the limit. This means that sequences of functions which are equivalent must have equivalent limits (up to non-degeneracy of the $\lambda$ sequence). This is particularly relevant for the study of adversarial training,
because it allows to alleviate the difficulty of the analysis,
through a reduction to the study of the limits
near $\delta \to 0$.
In particular, the adversarial gap admits as $\delta \to 0$ a limit
\begin{align}
\label{eq:advgap}
\mathcal{A}_\delta(\mathcal{D}; \theta) / \delta \, \underset{\delta \to 0}{\longrightarrow} \,
A(\mathcal{D}; \theta) := \mathbb{E}_{(x,y) \sim \mathcal{D}} \left[ \lVert \partial_x \mathcal{L}(y, x, \theta) \rVert_* \right] \quad \forall \theta \in \Theta.
\end{align}
Understanding the structure of the adversarial gap $\mathcal{A}_\delta(\mathcal{D}; \cdot)$ is challenging due to its dependency on both the attacker strength $\delta > 0$ and the data distribution $\mathcal{D}$. By considering the equivalence relation in the limit, we can alleviate one of these two sources of difficulty. Furthermore, the stability of the equivalence relation at the limits allows us to do so without any loss in generality.
Thus we can reject all families of candidates $(\Omega_\delta)_{\delta > 0}$  as soon as they admit as $\delta \to 0$ a limit $\Omega_\delta / \delta \to \Omega_+$
that is not equivalent to $A(\mathcal{D}; \cdot)$. This is formalized as Proposition~\ref{prop:convergencepropB}, with proof deferred to Appendix~\ref{appendix:definition}.

\begin{restatable}{proposition}{convergencepropB}
\label{prop:convergencepropB}
    Let $(\Omega_\delta \in \mathcal{F})_{\delta >0}$ be a family of regularizers
    such that $\Omega_\delta / \delta$ admits a limit $\Omega_+ \in \mathcal{F}$ as $\delta$ goes to $0$. If for every $\delta  >0$ there exists $c_\delta \in \mathbb{R}$ and $\lambda_\delta \in \mathbb{R}_+^*$ such that ${\Omega_\delta = \lambda_\delta \cdot \mathcal{A}_\delta(\mathcal{D}; \cdot) + c_\delta}$ with $\lambda_\delta = \Theta(1)$ near $\delta \to 0$, then  $\>\Omega_+$ is equivalent to $A(\mathcal{D}; \cdot)$, i.e. $\Omega_+ \sim_P A(\mathcal{D}; \cdot)$.
\end{restatable}

This proposition will be used through its contraposition, to disprove equivalence
by showing that limits are not equivalent.
The condition $\lambda_\delta = \Theta(1)$ prevents $\lambda_\delta$ from vanishing or diverging.
For instance, regularizers with $\Omega_\delta / \delta \to \Omega_+ = 0$ could satisfy the affine relationship for all $\delta$ by using a sequence $\lambda_\delta \to 0$, yielding inequivalent limits ${\Omega_+ = 0 \not\sim_P {A}(\mathcal{D}; \cdot)}$,
whereas controlled $(\lambda_\delta)_\delta$ ensure that equivalence transfers to the limits.
Note that ${\mathcal{A}_\delta(\mathcal{D}; \cdot) \to 0}$ as $\delta \to 0$,
since there is no adversarial gap when the attacker has vanishing capacity.
It is natural to require candidates $\Omega_\delta$ to satisfy $\Omega_\delta \to 0$ and admit a derivative with respect to~$\delta$, the simplest case being linear dependence $\Omega_\delta := \delta \cdot \Omega_+$

\textbf{Using the equivalence to reject candidate rewriting.}
In the remainder, we leverage this equivalence and its stability at the limit
to show that the adversarial gap cannot be rewritten as a
data-independent or weakly data-dependent regularization.
Specifically, we observe that the equivalence relation in~\eqref{eq:propdefinition} implies that equivalent functions satisfy two easily verifiable properties:
(i) \emph{matching minima} and (ii) \emph{aligned gradients}.
Hence, any candidate regularization $\Omega_+$ intended as a valid reformulation of $\mathcal{A}(\mathcal{D}; \cdot)$ must also satisfy both properties.
To show that a given candidate is not a suitable rewriting for the adversarial gap, it suffices to demonstrate that it violates at least one of these two properties.

\section{Rejecting regularizers by mismatch of minima}\label{sec:minima-displacement}

The first property we introduce for ruling out several candidate reformulations is the requirement that equivalent functions have \emph{matching minima}. Informally, this property states that if two functions are equivalent, then, although their values may differ, their level sets (and thus their minima) must coincide. This property is formalized in Proposition~\ref{prop:aligned-optima} and proved in \Appendix{aligned-optima-proof}.

\begin{restatable}[Equivalence implies matching restricted minima]{proposition}{alignedoptimaprop}\label{prop:aligned-optima}
  If $\Omega_1 \sim_P \Omega_2$
  then
  \[
    \forall \mathcal{U} \subseteq \Theta,
    \textstyle
    \quad \argmin_{\,\mathcal{U}} \Omega_1 = \argmin_{\,\mathcal{U}} \Omega_2
  \]
\end{restatable}

This simple observation is of both theoretical and practical relevance.
From a theoretical perspective, it means that it can be sufficient to analyze the ``shape'' of the set of minimizers
inside an arbitrary region $\mathcal{U}$ to disprove equivalence, instead of having to deal with function values
or global minima. We develop this idea in the case of two-layer neural networks in Section~\ref{sec:formalrejection1}. On the other hand, from a practical standpoint, it provides a tool for probing such properties in large neural networks,
via what we call \textit{slicing}. We further devlop the concept of slicing and its use in Section~\ref{sec:emprejection1}.

\subsection{Formal rejection of separable candidates for two-layer networks}
\label{sec:formalrejection1}

From a theoretical standpoint, we can use Proposition~\ref{prop:aligned-optima} to reject any data-independent neuron-wise separable candidate reformulation.
Specifically, for functions $\psi : \mathbb{R}^{m \times d} \to \mathbb{R}$ and
$\varphi_i : \mathbb{R}^K \to \mathbb{R}\,$ for all $i \in [m]$,
we consider the following set regularizers $\Omega_+ : \Theta \to \mathbb{R}$
over $\Theta = \mathbb{R}^{m \times d} \times \mathbb{R}^{m \times K}$.
\begin{equation}\label{eq:def-candidate-1}\tag{C1}
  \Omega_+ : (w,a) \mapsto \psi(w) + \sum_{i \in [m]} \varphi_i(a_i)
\end{equation}
This set of functions naturally includes simple sums of functions depending on only one parameter such as the the $L_1$ norm $\Omega_+(w,a) = \Vert w \Vert_1 + \Vert a \Vert_1$, the squared $L_2$ norm $\Omega_+(w,a) = \Vert w \Vert_2^2 + \Vert a \Vert_2^2 $, or more generally $\Omega_+(w,a) = \Vert w \Vert_p^p + \Vert a \Vert_p^p$. This class can also leverage the structure of weights more finely, for example using a $L_{2,1}$-norm on the weights of $a$, i.e. $\lVert a  \rVert_{2,1} := \sum_{i \in [m]} \lVert a_i \rVert_2$. This choice is not parameter-separable, but remains separable ``by neuron''. We formalize, in Proposition~\ref{prop:no-separability} below, the statement saying that candidates regularizer of this general form cannot be equivalent to the adversarial gap as introduced in~\eqref{eq:advgap}.

\begin{restatable}{proposition}{propnosep}\label{prop:no-separability}
  Let
  $\psi : \mathbb{R}^{m \times d} \to \mathbb{R}$
  and
   $\varphi_i : \mathbb{R}^{K} \to \mathbb{R}$
  for $i \in [m]$
  be continuous functions almost-everywhere differentiable,
  and define $\Omega_{+} : (w,a) \mapsto \psi(w) + \sum_{i \in [m]} \varphi_i(a_i)$
  as per \eqref{eq:def-candidate-1}.
  \linebreak[2]
  If $m \geq 3$,
  then there exists a distribution $\mathcal{D}$ such that
  ${\Omega_{+} \not\sim_P A(\mathcal{D}; \cdot)}$.
\end{restatable}

To establish this result, our main idea is to leverage Proposition~\ref{prop:aligned-optima}, that is, to construct a dataset $\mathcal{D}$ and a subset $\mathcal{U}$ for which $\mathcal{A}(\mathcal{D}; \cdot)$ does not share the same set of minimizers as any regularization that fits the class we just described. To this end, we rely on the observation that the set of minimizers of any separable function decomposes as the Cartesian product of the minimizers of its components,
i.e. for mappings $F : \mathcal{W} \times \mathcal{V} \to \mathbb{R}$ such that $F(\alpha, \beta) :=  g(\alpha) + h(\beta)$ for all for $(\alpha, \beta)  \in \mathcal{W} \times \mathcal{V}$, then
\begin{align}
\label{eq:separablefunction}
  \argmin_{(\alpha,\beta) \in \mathcal{W} \times \mathcal{V}} F(\alpha,\beta)
  = \left( \argmin_{\alpha \in\mathcal{W}} g(\alpha) \right)
  \times \left( \argmin_{\beta \in \mathcal{V}} h(\beta)  \right)
\end{align}
As, by construction, we know that any candidate regularizer we consider above satisfies this property, to disprove equivalence, it is sufficient to show that the adversarial gap cannot have this shape. We provide a more detailed proof sketch for Proposition~\ref{prop:no-separability} below.

\begin{tcolorbox}[boxrule=0pt,standard jigsaw, opacityback=0, frame hidden,sharp corners,enhanced,borderline west={1pt}{0pt}{gray},left=6pt,right=0pt,top=0pt,bottom=0pt,boxsep=1pt]
  \begin{proof}[Proof sketch]
    We proceed by contradiction. Let us assume $\Omega_+ \sim_P A(\mathcal{D}; \cdot)$ for all distributions $\mathcal{D}$.
    %Let us show that $\varphi_0$ is constant.
    Let $w_0 \in \mathbb{R}^{d} \setminus \{ 0 \}$
    and ${u = (w_0, w_0, w_1, \ldots, w_1) \in \mathbb{R}^{m \times d}}$,
    where $w_1 \in \mathbb{R}^d$ is such that $w_1 \cdot w_0 < 0$.
    Define
    ${\mathcal{U} = \{ u \} \times \mathcal{V}}$,
    where
    $\mathcal{V} = \mathbb{R}^K \times \mathbb{R}^K \times \{0\}^{(m-2) \times K} \subseteq \mathbb{R}^{m\times K}$.
    Let us fix $x = w_0$, $y \in \mathcal{Y}$,
    and $\mathcal{D} = \{ (x, y) \}$.
    Thus for any $\theta = (u,a) \in \mathcal{U}$ we have $f(\theta, x) = (a_0 + a_1)\, \sigma(w_0 \cdot x)$,
    and
    \[A(\mathcal{D}; \theta) = \lVert \partial_z \ell(y, f(\theta, x)) \cdot
    (a_0 \otimes w_0 + a_1 \otimes w_0) \rVert_*,  \]
    where $\partial_z$ denotes the partial derivative w.r.t the second component of $\ell$.
    When $a_1 = - a_0$, this yields
    $A(\mathcal{D}; \theta) = 0$. In particular, this means that ${(u, (a_0, -a_0, 0 \ldots)) \in \argmin_{\theta \in \mathcal{U}} A(\mathcal{D}; \theta)}$.
    But by \Proposition{aligned-optima},
    $\argmin_{\theta \in \mathcal{U}} A(\mathcal{D}; \theta) = \argmin_{\theta \in \mathcal{U}} \Omega_+ (\theta)$. Furthermore, by separability of the candidate,
    $\argmin_{\, \mathcal{U}} \Omega_+
    = \{ u \} \times \left( \argmin \varphi_0 \times \argmin \varphi_1 \times \{0\} \right)$, hence we have ${ a_0 \in \argmin \varphi_0 }$ and ${ (-a_0) \in \argmin \varphi_1}$. Since this observation is valid for all choices $a_0 \in \mathbb{R}^k$, we get $(\argmin \varphi_0) = \mathbb{R}^K = (\argmin \varphi_1)$, therefore both $\varphi_0$ and $\varphi_1$ are constant.

\medskip
    Finally, writing $\gamma = \sigma(w_0 \cdot x) \in \mathbb{R}_+^*$
    it suffices to choose $a_0 \in \mathbb{R}^K$ such that $\partial_z \ell(y, \gamma\,a_0) \cdot a_0 \neq 0$ (a construction is shown in \Lemma{nontriviality-ce-grad}), and
    observe that ${\theta := (u, (a_0, 0, \ldots 0)) \in \argmin_\mathcal{U} \Omega_+}$
    but $A(\mathcal{D}; \theta) = \lvert \partial_z \ell(y, \gamma\, a_0) \cdot a_0 \rvert \cdot \lVert w_0 \rVert_* \neq 0 = \min_\mathcal{U} A(\mathcal{D}; \cdot)$ thus $\theta \notin \argmin_\mathcal{U} A(\mathcal{D}; \cdot)$.
    Thus minima do not coincide
    ${[\,\argmin_\mathcal{U} \Omega_1 \neq \argmin_\mathcal{U} A(\mathcal{D}; \cdot)\,]}$,
    contradicting \Proposition{aligned-optima}.
\end{proof}
\end{tcolorbox}

\subsection{Empirical rejection of separable candidates for more general architectures}
\label{sec:emprejection1}

From the empirical standpoint, the same type of separability argument can be used to disprove equivalence between simple regularizers and the adversarial gap
in much broader settings. Indeed, when a function $F$ is separable, it satisfies~\eqref{eq:separablefunction},
hence for any two origin points $\beta_0$ and $\beta_1$, we get
\[ \argmin_{(\alpha, \beta) \in \mathcal{W} \times \{ \beta_k \}} F(\alpha, \beta)
= \left( \argmin_{\alpha \in \mathcal{W}} g(\alpha) \right) \times \{ \beta_k \} \quad\text{ for all } k \in \{0,1\}\]
Therefore, if we consider the sets $\mathcal{U}_0= \mathcal{W} \times \{ \beta_0 \}$ and $\mathcal{U}_1 = \mathcal{W} \times \{ \beta_1 \}$,
which we call ``slices'',
the left projection of both
$(\argmin_{\,\mathcal{U}_0} F)$ and $(\argmin_{\,\mathcal{U}_1} F)$
are identical (they are both exactly $\argmin_\mathcal{W} g$)
without any assumption on $g$. This observation is useful because empirically we can use this property to verify the equivalence property at a reasonable cost if $(\dim \mathcal{W})$ is small.

\textbf{Full separability}.
The simplest example is full separability
such as $\Omega_+(w,a) = \lVert w \rVert_p^p + \lVert a \rVert_p^p$,
which is separable with respect to each parameter individually,
thus we can choose $(\dim\mathcal{W}) = 1$
and $\mathcal{V} =\mathbb{R}^{n-1}$ where $n = m\, (d+k)$ for the remaining parameters.
Then, freezing all parameters but one, we can numerically observe the difference between $\Omega_+$ and $A(\mathcal{D}, \cdot)$ on the slice.
In this case, we can easily check (visually or numerically) whether both minima coincide on the slice $\mathcal{W} \times \{ \beta \}$.

\textbf{Weaker separability.} In the example of neuron-separability, visualization is a little more complicated. Indeed $\Omega_+(w,a) = \psi(w) + \sum_{i \in [m]} \varphi_i(a_i)$, each $\varphi_i : \mathbb{R}^K \to \mathbb{R}$
gives rise to a separability with $(\dim \mathcal{W}) = K$ corresponding to a single
$a_i$ for some $i \in [m]$.
This is a problem for visualizations if $K$ is large, but can be circumvented by choosing
a lower-dimensional space to compute the argmin.
Indeed, it suffices to pick an affine subspace of dimension $q \in \{ 1, 2\}$
in $\mathcal{W}$, that is to say
$\mathcal{W}_0 = \{ u_0 + \sum_{i \in [q]} \alpha_i \,v_i \mid \alpha \in \mathbb{R}^q \}$,
for directions $(v_i)_{i \in [q]} \in \mathcal{W}^q$ and an origin $u_0 \in \mathcal{W}$.
\linebreak[2]
By the same argument, we can check if the left projection of the argmin remains constant
between slices.
Note that in this case the origin $u_0 \in \mathcal{W}$ must be held fixed as $\beta \in \mathcal{V}$ varies.
In the case of neuron separability, this means that the weights $a_{i} \in \mathbb{R}^K$
of the $i$-th neuron must be held fixed, while the rest of the parameters
$\beta = (w \in \mathbb{R}^{m \times d}, (a_j)_{j\neq i} \in \mathbb{R}^{(m-1) \times K})$
can be varied.
We can plot the adversarial gap on each slice $\mathcal{W}_0 \times \{ \beta \}$,
where each choice of $\beta$ corresponds to a new slice.
For instance $\mathcal{W}_0 = \mathbb{R} \times \mathbb{R} \times \{ (a_{i,j})_{j \in [K] \setminus \{0,1\}} \} \subset \mathbb{R}^K$
can be visualized as a heatmap.

\textbf{Experimental verification in two-layer setting.}
To confirm by experiments the existence of the displacement of minima proven above
in two-layer networks, with data distributions more similar to standard configurations,
we perform adversarial training via 10-step PGD of
a two-layer network of width $m = 10^3$ (with bias weights) on CIFAR-10 classification \citep{krizhevsky2009learning},
After $t_0 = 10$ epochs, with freeze the first neuron's output weights $a_0 \in \mathbb{R}^K$,
and observe the value of $A(\mathcal{D}; \cdot)$ on the slice
$\mathcal{W}_0 \times \{ \beta{(t)} \}$ for $t \in \{ 10, 20, 30, 40, 50 \}$ the number of epochs
of training, where
$\mathcal{W}_0 = \{ a_0{(t_0)} + \alpha \, v_0 \mid \alpha \in \mathbb{R} \} \subset \mathbb{R}^K$
is a one-dimensional slice using $a_0{(t_0)}$ the frozen first neuron's weight,
$v_0$ a random direction in $\mathbb{R}^K$ drawn from $\mathcal{N}(0, I_K / K)$, and $\beta{(t)}$ denotes the rest of learned weights
at epoch $t$.
Results are shown in \Figure{projected-2lp}, and show that the left-projection
of the argmin is not constant as training progress. Therefore, even for the specific CIFAR-10 dataset, no separable regularization can be equivalent to the adversarial gap.

\textbf{Experimental rejection of separability in Wide ResNets.}
To push the experiment to more realistic architectures, we perform a similar
slice comparison with Wide-ResNet of depth 34 and 10\textrm{x} widening
\citep{zagoruyko2016wide}
on CIFAR-10 classification.
Although our proof only rigorously covers two-layer networks, the tools leveraged
are only the ability to turn on different sets of neurons, which seems reasonable
even in deeper architectures, and as such we would expect this impossibility of
equivalence to extend naturally.
To confirm this intuition, we perform adversarial training on this larger for
$t_0 = 10$ epochs with 10-step PGD, then freeze the first neuron of the last layer
and continue training for 10 more epochs.
We observe two slices $\mathcal{W}_0 \times \{ \beta(t_0) \}$ and
$\mathcal{W}_0 \times \{ \beta(t_1) \}$ for $t_1 = 20$ and
$\mathcal{W}_0 = \{ a_0(t_0) + \alpha_0 v_0 + \alpha_1 v_1 \mid (\alpha_0, \alpha_1) \in \mathbb{R}^2\} \subset \mathbb{R}^K$ where $v_i = (\mathds{1}_{j=i})_{j \in [K]}$ for $i \in \{0,1\}$
are the first two canonical basis vectors of $\mathbb{R}^K$.
The value of $A(\mathcal{D}; \cdot)$ on the slices is shown in \Figure{wideresnet-E10} and \Figure{wideresnet-E20} (each with 51\textrm{x}51 pixels for $(a_{0,i}(t_0) + \alpha_i v_i) \in [-0.5, +0.5]$,
frozen weight ${(a_{0,i}(t_0))_{i \in \{0,1\}}} \in \mathbb{R}^2$ corresponding to ${\alpha = (0,0)}$ depicted by a white cross).
Again we observe that the minimum of $A(\mathcal{D}; \cdot)$ is not constant when $\beta$ varies, thus no separable regularization can be equivalent to the adversarial gap with this combination of dataset and architecture either.

\begin{figure}[H]
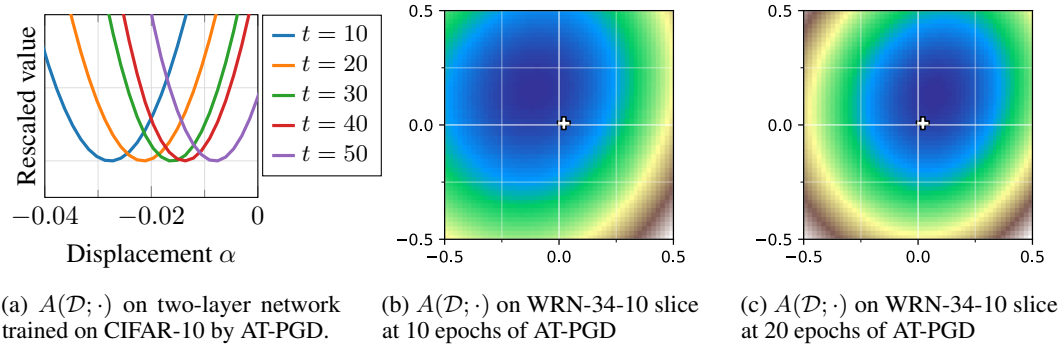

  \centering
  \begin{subfigure}{.32\textwidth}
      \begin{tikzpicture}
        \begin{axis}[
            width=4.4cm,
            height=4.0cm,
            xmin=-0.04, xmax=0.0,
            xtick={-0.04, -0.02, 0.00, 0.02},
            minor x tick num = 1,
            scaled x ticks=false,
            ymin = -0.005, ymax = +0.02,
            ymajorticks = false,
            grid=both,
            grid style={draw=gray!25},
            xticklabel style={
                /pgf/number format/fixed,
                /pgf/number format/precision=3,
            },
            legend style={font=\small, at={(1.02,+1.0)}, anchor=north west},
            xlabel = {Displacement $\alpha$},
            ylabel = {Rescaled value},
            legend image code/.code={
                \draw[mark repeat=2,mark phase=2]
                plot coordinates {
                    (0cm,0cm)
                    (0.15cm,0cm)        %% default is (0.3cm,0cm)
                    (0.3cm,0cm)         %% default is (0.6cm,0cm)
                };%
            }
        ]
        \addplot+[solid, tab1, very thick, mark=none] table [x=scalar, y=E10, col sep=comma]{mlp_r0.02_k50_D12.csv};
        \addplot+[solid, tab2, very thick, mark=none] table [x=scalar, y=E20, col sep=comma]{mlp_r0.02_k50_D12.csv};
        \addplot+[solid, tab3, very thick, mark=none] table [x=scalar, y=E30, col sep=comma]{mlp_r0.02_k50_D12.csv};
        \addplot+[solid, tab4, very thick, mark=none] table [x=scalar, y=E40, col sep=comma]{mlp_r0.02_k50_D12.csv};
        \addplot+[solid, tab5, very thick, mark=none] table [x=scalar, y=E50, col sep=comma]{mlp_r0.02_k50_D12.csv};

        \legend{$t=10$, $t=20$, $t=30$, $t=40$, $t=50$};

        \end{axis}
      \end{tikzpicture}
      \caption{$A(\mathcal{D}; \cdot)$ on two-layer network trained on CIFAR-10 by AT-PGD.}\label{fig:projected-2lp}
  \end{subfigure}\hfill{}%
  \begin{subfigure}{.30\textwidth}
    \centering
    \resizebox{.99\textwidth}{!}{\includesvg{wrn_train_E0000_s0.50_Tavg_norm_pix.svg}}
    \caption{$A(\mathcal{D}; \cdot)$ on WRN-34-10 slice at 10 epochs of AT-PGD}\label{fig:wideresnet-E10}
  \end{subfigure}\hfill{}%
  \begin{subfigure}{.30\textwidth}
    \centering
    \resizebox{.99\textwidth}{!}{\includesvg{wrn_train_E0010_s0.50_Tavg_norm_pix.svg}}
    \caption{$A(\mathcal{D}; \cdot)$ on WRN-34-10 slice at 20 epochs of AT-PGD}\label{fig:wideresnet-E20}
  \end{subfigure}
  \caption{
  Level sets of $A(\mathcal{D}; \cdot)$ induced by the $\ell_\infty$-norm
  on CIFAR-10 with a two-layer network (left) and Wide ResNet (right).
  The first output neuron of each network is frozen after 10 epochs to
  allow for comparable slices.
  Learned weights in Wide ResNet indicated by a white cross.}
  \label{fig:resnet}
\end{figure}

\section{Rejecting based on misalignment of gradients}\label{sec:gradient-misalignment}

The restriction to subsets $\mathcal{U} \subsetneq \Theta$ in Proposition~\ref{prop:aligned-optima} enables the construction of simple counterexamples and facilitates a visual search for non equivalence in more general architectures. By further localizing the analysis to neighborhoods of a single point $\theta \in \Theta$, we can rule out significantly larger classes of candidates through analytical arguments. The second property we introduce for excluding broader classes of candidate regularization is the fact that equivalent functions have aligned gradients. This property is formalized in Proposition~\ref{prop:cosine}, and its complete proof is deferred to \Appendix{aligned-gradients-proof}.

\vspace{.1cm}
\begin{restatable}[Equivalence implies local alignment of gradients]{proposition}{alignedgradientsprop}\label{prop:cosine} Let  $n = \text{dim}(\Theta)$ be the total number of parameters.
  If $\Omega_1 \sim_P \Omega_2$ then there exists a set $S \subset \Theta$ of Lebesgue
  measure zero\footnote{
    The excluded set $S$ corresponds to points where either gradient may be undefined, for instance if $\Omega_1 = \lVert \cdot \rVert_1$
    then $\{ u \in \mathbb{R}^n \mid \exists i \in [n], u_i = 0\} \subseteq S$.
    Functions in $\mathcal{F}$ are continuous, so evaluation can be extended by continuity, but gradients cannot, they may be undefined
    on a small (Lebesgue measure zero) set.
  },
  such that  for all $Q \in \mathbb{R}^{n \times n}$ symmetric positive semi-definite (PSD), it holds
  \[
  \forall \theta \in \Theta \setminus S, \quad
  \langle \,\nabla \Omega_1(\theta), \nabla \Omega_2(\theta) \rangle_Q^2
  = \lVert \nabla \Omega_1(\theta) \rVert_Q^2 \,
  \lVert \nabla \Omega_2(\theta) \rVert_Q^2
  \]
\end{restatable}

This proposition states that two equivalent functions satisfy a saturated Cauchy-Schwarz inequality with respect to the Mahalanobis semi-norm induced by $Q$. This can be interpreted as an alignment of gradients in the case where $Q = I_n$, or more generally as an alignment of the projections of the gradients when $Q \neq I_n$. This proposition is useful for our theoretical analysis as it will allows us to reject all data independent as well as all the loss-modulated candidates we mentioned in Section~\ref{sec:intro}.
For instance $Q((w,a) , (w', a')) := \sum_{i \in [m]} \sum_{j \in [K]} a_{i,j} a'_{i,j}$ restricts attention to the second-layer.

\subsection{Rejecting data-independence via gradient alignment}
\label{sec:rejecting-data-independence}

Leveraging this more sophisticated tool, we can dismiss a much larger class of candidate regularizers.
\begin{equation}\label{def-candidate-2}\tag{C2}
  \text{Any data-independent function} \quad \Omega_+ : \Theta \to \mathbb{R}
\end{equation}
Data-independence alone is sufficient to show that no such function can be equivalent to adversarial training, under mild assumptions of having at least $d \geq 3$ input dimensions and $m \geq 3$ neurons.

\smallskip
\begin{restatable}{proposition}{propnodataindependence}\label{prop:no-data-independence}
    Let $\Omega_+ : \Theta \to \mathbb{R}$.
    If $m \geq 3$ and $d \geq 3$, there exists $\mathcal{D}$ such that
    $\Omega_+ \not\sim_P A(\mathcal{D}; \cdot)$
\end{restatable}

To establish this, we construct distributions $\mathcal{D}_0$ and $\mathcal{D}_1$
such that ${A(\mathcal{D}_0; \cdot) \not\sim_P A(\mathcal{D}_1; \cdot)}$.
Thus $\Omega_+$ cannot be equivalent to both by transitivity.
This is done by crafting data distributions from special inputs $x_i$
to ``turn on'' exactly one neuron $(w_i,a_i)$,
thus the gradient of the adversarial gap is supported on parameters
of this specific neuron.
Inequivalence follows by Proposition~\ref{prop:cosine}, since gradients with disjoint support cannot be aligned.
Details are a little technical and leverage the two-layer network structure to keep computations analytically manageable, the proof is deferred to \Appendix{no-data-independence-proof}.

\begin{tcolorbox}[boxrule=0pt,standard jigsaw, opacityback=0, frame hidden,sharp corners,enhanced,borderline west={1pt}{0pt}{gray},left=6pt,right=0pt,top=0pt,bottom=0pt,boxsep=1pt]
  \begin{proof}[Proof sketch]
    We sketch only a strict violation of the equality at a single point.
    Extraction of an open set around it where gradients are well defined and violate the condition
    is deferred to appendix.

    \medskip

    Define $w = (w_0, w_1, u, u, \ldots, u) \in \mathbb{R}^{m \times d}$ with
    $(w_0, w_1, u)$ orthonormal and any $a \in \mathbb{R}^{m \times K}$ to form $\theta = (w,a)$.
    Let $x_0 = w_0 - w_1 - u$. For this $x_0$ we have $\sigma(w_0 \cdot x_0) > 0$
    but $\sigma(w_i \cdot x_0) = 0$ for $i \neq 0$, and thus $f(\theta, x_0) = a_0 \,\sigma(w_0 \cdot x_0)$. This means that $x_0$ only ``turns on'' the first neuron of $f$. We can build $x_1$ symmetrically to ``turn on'' the second neuron only.
  \medskip

    Construct distributions $\mathcal{D}_\alpha$ with
    $\{ (x_0, y_0), (x_1, y_1) \}$ and weights ${\alpha_0 \in \interval{0}{1}}$ and ${\alpha_1 =1 - \alpha_0}$,
    having a linearized adversarial gap
    $ {A(\mathcal{D}_\alpha; \theta) = \sum_{i \in \{0,1\}} \alpha_i \, \lVert \partial_z \ell(y_i, f(\theta, x_i)) \cdot \partial_x f(\theta, x_i)\rVert_*}$
    %\vspace{-7pt}
    which is an $\alpha$-weighted average
    of $A(\mathcal{D}_0; \theta)$ and $A(\mathcal{D}_1; \theta)$.
    By construction of $(x_0, x_1)$, $f(\theta, x_i)$ and thus $\partial_x f(\theta, x_i)$
    depend only on $(w_i, a_i)$,
    so the vectors
    $\partial_\theta A(\mathcal{D}_0; \theta)$
    and $\partial_\theta A(\mathcal{D}_1; \theta)$
    are supported respectively on the subspaces $(\partial w_0, \partial a_0)$ and $(\partial w_1, \partial a_1)$,
    which are orthogonal subspaces, since they correspond to disjoint sets of neurons.
    In particular those vectors cannot be aligned. %, which concludes the proof.
  \end{proof}
\end{tcolorbox}

When $\alpha$ varies, the average
$\partial_\theta A(\mathcal{D}_\alpha; \cdot)$ ``rotates'',
thus cannot be aligned with a data-independent vector field.
The proof is restricted to two-layer networks, but this
is naturally a very general construction, which will not be limited
to such networks.
There is exactly one exception to this kind of reasoning:
if the fields induced by all samples are identical,
i.e. if $\partial_x f(\theta, x)$ is independent of $x$,
as in the bilinear models
$(\theta, x) \mapsto \theta \cdot x$ studied in \citet{blanchet2019robust,ribeiro2023regularization}.

\subsection{Rejecting loss-modulated candidates via gradient alignment}
\label{sec:rejecting-loss-modulation}

We have shown that \emph{data-independent}
additive corrections cannot be equivalent to adversarial training.
However, \citet{blanchet2019robust} and \citet{ribeiro2023regularization}
resorted in the case of linear models to
equivalence with a loss-modulated\footnote{The term \textit{modulation} refers to the fact that this
surrogate tends to the original loss when $\delta \to 0$, but the correction for
positive $\delta$ happens \textit{inside} the loss, causing a slightly data-dependent
alteration depending on $\lVert \theta \rVert_*$.}
regularization of the form
${\mathbb{E}[ \,L(y \, \langle x, \theta \rangle - \delta \lVert \theta \rVert_*)\,]}$.
This kind of regularization is equivalent near $\delta \to 0$ to an additive \emph{data-dependent} correction of the form
${- \delta \cdot \mathbb{E}[\, \lvert \nabla L(y \, \langle x, \theta \rangle) \rvert \,] \cdot \lVert \theta \rVert_*}$.
This leads us to consider the class of loss-modulated candidates defined
for any $H : \Theta \to \mathbb{R}^{K \times h}$
and any convex $\phi : \mathbb{R}^h \to \mathbb{R}$ as
\begin{equation}\label{eq:def-candidate-3}\tag{C3}
    \Omega_+(\mathcal{D}; \theta) = \mathbb{E}_{(x,y) \sim \mathcal{D}}\Bigl[
    \phi \Bigl( \partial_z \ell(y, f(\theta, x))
    \cdot H(\theta) \Bigr)
    \Bigr]
\end{equation}
The simpler case $K=1$ and $\phi = \lVert \cdot \rVert_*$ recovers
the form of \citet{ribeiro2023regularization} because
the loss derivative is a scalar
$\partial_z \ell(y, u) \in \mathbb{R}^K = \mathbb{R}$
which can be factored out of the norm, yielding
${\mathbb{E}_{\mathcal{D}}\left[
\lvert \partial_z \ell(y, f(\theta, x)) \rvert \right]
\cdot \phi (H(\theta))}$ hence the parameter-shrinking behavior
when ${\phi(H(\theta)) = \lVert \theta \rVert_*}$.
This reduction is not possible for $K > 1$,
even for linear models.
The parameter-shrinking form $\lVert \theta \rVert_*$
is also restricted to \textit{bilinear} models $f(\theta, x) = \langle x, \theta\rangle$.
The equivalent regularization for
\textit{data-linear} models $f(\theta, x) = \langle x, \Phi(\theta) \rangle$ is
$\lVert \partial_x f \rVert_* = \lVert \Phi(\theta) \rVert_*$.
These are distinct from \textit{parameter-linear}
models $f(\theta, x) = \langle \Psi(x), \theta \rangle$
studied in convex methods, the two types of linearity are
structurally different.

In general, the linearized adversarial
gap is
${A(\mathcal{D}; \theta) = \mathbb{E} [\lVert \partial_z \ell(y, f(\theta, x)) \cdot \partial_x f(\theta, x) \rVert_*]}$.
As previously observed, the data-dependence
of $\partial_x f(\theta, x)$ can only be erased if $f(\theta, x)$ is linear
in $x$ and thus has constant derivative with respect to $x$.
The open question is whether there exists \textit{another}
clever data-independent quantity ${H(\theta) \in \mathbb{R}^{K \times h}}$
such that
% $\theta \mapsto \mathbb{E}[ \, \lVert \partial_z \ell(y,f(\theta, x)) \cdot H(\theta) \rVert \, ]$
$\Omega_+(\mathcal{D}; \cdot)$
is equivalent to $A(\mathcal{D},\cdot)$
even with non-constant gradient, which could yield
a reduced training cost.
We show that there cannot be any such function.
\linebreak[1]
Specifically, we construct for any $(H, \phi)$ a distribution $\mathcal{D}$
with $\Omega_+(\mathcal{D}; \cdot) \not\sim_P A(\mathcal{D}; \cdot)$.
Note that $\Omega_+$ \textit{does} depend on $\mathcal{D}$, and thus escapes the previous
data-independent rejection by transitivity.

\begin{restatable}{proposition}{propnomodulation}\label{prop:no-modulation}
    For $h \geq 1$, let
    ${H : \Theta \to \mathbb{R}^{K \times h}}$ and
    ${\phi : \mathbb{R}^h \to \mathbb{R}}$ be continuous functions
    almost-everywhere differentiable, with $\phi$ convex,
    defining
    ${\Omega_+(\mathcal{D}; \theta) = \mathbb{E}_\mathcal{D}
    [ \phi\left( \partial_z \ell(y, f(\theta, x)) \cdot H(\theta) \right)] }$
    as per \eqref{eq:def-candidate-3}.
    \linebreak[1]
    If $m \geq 3$, and $d \geq 3$
    then there exists $\mathcal{D}$ such that
    ${\Omega_+(\mathcal{D}; \cdot) \not\sim_P A(\mathcal{D}; \cdot)}$.
\end{restatable}
The complete proof is deferred to \Appendix{no-modulation-proof}.
All major ideas and derivations are presented in the proof sketch below,
but technicalities on differentiability and non-triviality are delegated to appendix.

\begin{tcolorbox}[boxrule=0pt,standard jigsaw, opacityback=0, frame hidden,sharp corners,enhanced,breakable,borderline west={1pt}{0pt}{gray},left=6pt,right=0pt,top=0pt,bottom=0pt,boxsep=1pt]
  \begin{proof}[Proof sketch]
    By contradiction,
    assume that $\Omega_+(\mathcal{D}; \cdot) \sim_P A(\mathcal{D}; \cdot)$ for all distributions $\mathcal{D}$,
    and let us construct a distribution for which the equality of \Proposition{cosine} is violated.
    We sketch strict violation at a point, extraction of an open set with well-defined
    gradients is deferred to appendix.

\medskip
    Let $(w_0, w_1)$ be two orthonormal vectors in $\mathbb{R}^d$, and extend them into
    $u \in \mathbb{R}^{m \times d}$ such that $u_0 = w_0$ and $u_1 = w_1$, together with
    points $(x_k \in \mathbb{R}^d)_{k \in \{0,1\}}$
    such that for all $j \in [m]$ and $k \in \{0,1\}$,
    \linebreak[2] if $j \neq k$ then $\sigma(u_j \cdot x_k) = 0$.
    For shortness, we write
    $\sigma_s := \sigma(w_s \cdot x_s) \in \mathbb{R}_+^*$ for $s \in \{0,1\}$.

\medskip
    For some $y \in [K]$,
    define the distributions
    $\mathcal{D}_0 = \{ (x_0, y) \}$ and $\mathcal{D}_1 = \{ (x_1, y) \}$.
    By the equivalence assumption, there exists $(\lambda_s \in \mathbb{R}_+^*)_{s \in \{0,1\}}$ such that
    $\forall s \in \{0,1\}, \>
    \partial_\theta \Omega_+(\mathcal{D}_s; \cdot) = \lambda_s \, \partial_\theta A(\mathcal{D}_s; \cdot)$.

\medskip
    Define $r_s : a \mapsto \partial_z \ell(y, f((u,a), x_s)) = \partial_z \ell(y, a_s\,\sigma_s) \in \mathbb{R}^K$,
    having ${\partial_{a_j} r_s(a) = 0}$ when $j \neq s$.
    By definition, $\Omega_+(\mathcal{D}_s, (u,a)) =\phi(r_s(a) \cdot H(u,a))$
    and
    ${ A(\mathcal{D}_s, (u,a))
    % = \psi( r_s(a) \cdot \partial_x f((u,a), x_s))
    = \psi(r_s(a) \cdot (a_s \otimes w_s)) }$.
    Thus compute for any $j \in [m]$ and $s \in \{0,1\}$ the derivatives
    %
    % \begin{equation}\label{eq:short-derivative-computation}\tag{E2}\begin{aligned}
    $$\begin{aligned}
    \frac{ \partial \Omega_+ }{\partial a_j} (\mathcal{D}_s; (u,a))
    = \Bigl( \partial_{a_j} r_s(a) \cdot H(u,a) + r_s(a) \cdot \partial_{a_j} H(u,a) \Bigr)
    &\cdot \nabla \phi\bigl(r_s(a) \cdot H(u,a)\bigr)
    \\
    \frac{\partial A}{\partial a_j} (\mathcal{D}_s; (u,a))
    = \mathds{1}_{j=s} \, \Bigl(\partial_{a_s} r_s(a) \cdot (a_s \otimes w_s) + r_s(a) \otimes w_s \Bigr)
    &\cdot \nabla \psi\bigl(r_s(a) \cdot (a_s \otimes w_s)\bigr)
    %\end{aligned}\end{equation}
    \end{aligned}$$

    For all $j \in \{0,1\}$, let us decompose
    $\partial_{a_j} \Omega_+(\mathcal{D}_s; \theta) \in \mathbb{R}^K$
    into its left and right terms
    $$L^\Omega(j,s) := (\partial_{a_j} r_s(a) \cdot H(\theta) ) \cdot \nabla \phi(r_s(a) \cdot H(\theta)) \in \mathbb{R}^K$$
    $$R^\Omega(j,s) := (r_s(a) \cdot \partial_{a_j} H(\theta) )  \cdot \nabla \phi(r_s(a) \cdot H(\theta)) \in \mathbb{R}^K$$

    The structure of $L^\Omega$ is diagonal, similarly
    to $\partial_{a_j} A(\mathcal{D}_s; \theta)$.
    However $R^\Omega$ can have a distinct rank-one structure which will be leveraged to prove the contradiction distinguishing $\Omega_+$ and $A$.

    \medskip

    Indeed by previously observed cancellation $\partial_{a_j} r_s(a) = 0$ when $j \neq s$, thus $L^\Omega$ is null off-diagonal.
    However, $R^\Omega$ has a different structure. For a careful choice of $a$
    satisfying $r_s(a) = r_j(a)$
    and $\lVert \partial_{a_s} A(\mathcal{D}_s; (u,a)) \rVert_2^2 \neq 0$,
    we have immediately that
    $R^\Omega(j, s) = R^\Omega(j,j)$ for all $(j,s)$,
    and we can show by calculations leveraging the equivalence assumption
    that $R^\Omega(j,j) \neq 0 \in \mathbb{R}^K$.
    \linebreak[2]
    In particular it follows when $j \neq s$ that
    $\partial_{a_j} \Omega_+(\mathcal{D}_s; \theta) = L^\Omega(j,s) + R^\Omega(j,s) = 0 + R^\Omega(j,j) \neq 0$.

\medskip
    Then, using a PSD form $Q$ projecting to the subspace $(\partial a_0, \partial a_1)$,
    we get by definition
    ${ \pnorm{Q}{\partial_\theta \Omega_+(\mathcal{D}_s; \theta)}^2
    = \sum_{j \in \{0,1\}} \lVert \partial_{a_j} \Omega_+(\mathcal{D}_s; \theta) \rVert_2^2 } $
    and by canceling off-diagonal ($j\neq s$) terms $\partial_{a_j} A(\mathcal{D}_s; \theta) = 0$
    that
    $ { \pnorm{Q}{ \partial_\theta A(\mathcal{D}_s; \theta) }^2
    = \lVert \partial_{a_s} A(\mathcal{D}_s; \theta) \rVert_2^2}$.
    Hence by Cauchy-Schwarz,
    $$\begin{aligned} \langle \partial_\theta \Omega_+(\mathcal{D}_0; \theta), \partial_\theta A(\mathcal{D}_0; \theta) \rangle_Q
        &= \partial_{a_0} \Omega_+(\mathcal{D}_0; \theta) \cdot \partial_{a_0} A(\mathcal{D}_0; \theta)
        + \partial_{a_{1}} \Omega_+(\mathcal{D}_0; \theta) \cdot \partial_{a_{1}} A(\mathcal{D}_0; \theta)
        \\ &= \partial_{a_0} \Omega_+(\mathcal{D}_0; \theta) \cdot \partial_{a_0} A(\mathcal{D}_0; \theta) + 0
         \\ &\leq \lVert \partial_{a_0} \Omega_+(\mathcal{D}_0; \theta) \rVert_2 \cdot \lVert \partial_{a_0} A(\mathcal{D}_0; \theta) \rVert_2
    \end{aligned}$$
    Thus, using this inequality on the numerator and canceling the factors
    $\lVert \partial_{\theta} A(\mathcal{D}_0; \theta) \rVert_Q^2 \neq 0$,
    $$\begin{aligned}
        \frac{ \langle \partial_\theta \Omega_+(\mathcal{D}_0; \theta), \partial_\theta A(\mathcal{D}_0; \theta) \rangle_Q^2}
        {\lVert \partial_\theta \Omega_+(\mathcal{D}_0; \theta) \rVert_Q^2 \, \lVert \partial_\theta A(\mathcal{D}_0; \theta) \rVert_Q^2}
        \leq \frac{ \lVert \partial_{a_0} \Omega_+(\mathcal{D}_0; \theta) \rVert_2^2 }{ \lVert \partial_{a_0} \Omega_+(\mathcal{D}_0; \theta) \rVert_2^2 + \lVert \partial_{a_{1}} \Omega_+(\mathcal{D}_0; \theta) \rVert_2^2 }
        < 1
    \end{aligned}$$
    where the strict inequality follows from the fact that $\partial_{a_1} \Omega_+(\mathcal{D}_0; \theta) \neq 0$ as shown above, thus the alignment equality of \Proposition{cosine} is violated at $\theta$, which concludes the contradiction.
  \end{proof}
\end{tcolorbox}

\paragraph{Empirical rejections by gradient alignment}

The previous arguments formally reject equivalence for a broad class of functions
in two-layer networks.
But these tools can also be used experimentally to disprove equivalence on targeted candidates across architectures.
In \Figure{wrn-cosine}, we measure the cosine similarity
(projected to first or last layer) between the linearized adversarial gap gradient
$\partial_\theta A(\mathcal{D}; \theta_t)$
and a regularizer gradient $\partial_\theta \mathcal{R}(\theta_t)$ along an adversarial training trajectory $(\theta_t)_{t \in \mathbb{N}}$ for WideResNet-34-10 on {CIFAR-10}.
Since this cosine is strictly below one, this experiment proves
that there is no equivalence between the two.
Two choices of $\mathcal{R}$ are pictured:
the $L_2$ and $L_1$ norm.
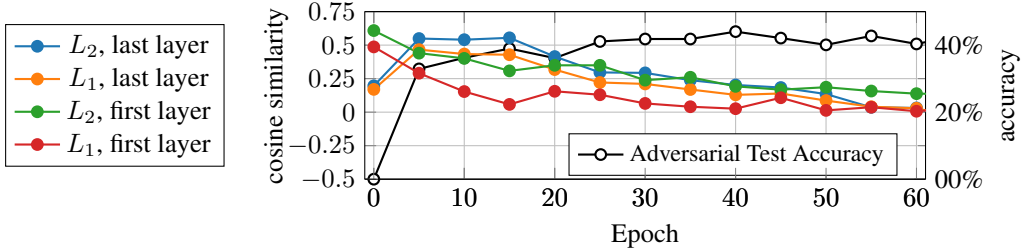
\begin{figure}[H]
    \centering
    \begin{tikzpicture}
        \begin{axis}[
            grid=both,
            legend style={ at={(0.95,+0.02)}, anchor=south east, font=\small},
            xmin=-1, xmax=61,
            ymin=0.0, ymax=0.5,
            ytick={0, 0.2, 0.4, 0.6, 0.8, 1.0},
            minor y tick num=1,
            yticklabels={00\%, 20\%, 40\%, 60\%, 80\%, 100\%},
            ylabel={accuracy},
            axis y line*=right,
            width=9cm,
            height=3.8cm,
            y label style={at={(+1.12,0.5)}},
        ]

        \addplot[color=black, mark=*, mark options={fill=white}, mark size=2pt, thick] table[
            x=epoch,
            y=adv_accuracy,
            col sep=comma,
            discard if not={model}{wrn-sgd-wd-sched-large3-E1p00},
        ] {wrn-sgd_accuracy.csv};
        \addlegendentry{Adversarial Test Accuracy}
        \end{axis}

        \begin{axis}[
            xlabel={Epoch},
            legend style={
                    at={(-0.25,+0.5)},
                    anchor=east,
                    legend columns=1,
                    },
            grid=none,
            axis y line*=left,
            ymin=-0.5, ymax=+0.75,
            ytick={-1.0, -0.75, -0.5, -0.25, 0, +0.25, +0.5, +0.75, +1.0},
            xmin=-1, xmax=61,
            width=9cm,
            height=3.8cm,
            ylabel={cosine similarity},
            y label style={at={(-0.12,0.5)}},
        ]

        \addplot[color=tab1, mark=*, mark size=2pt, thick] table[
            x=epoch,
            y={l2_cosine_last},
            col sep=comma,
        ] {wrn-sgd-wd-sched-large3_adv.csv};
        \addlegendentry{$L_2$, last layer}

        \addplot[color=tab2, mark=*, mark size=2pt, thick] table[
            x=epoch,
            y={l1_cosine_last},
            col sep=comma,
        ] {wrn-sgd-wd-sched-large3_adv.csv};
        \addlegendentry{$L_1$, last layer}

        \addplot[color=tab3, mark=*, mark size=2pt, thick] table[
            x=epoch,
            y={l2_cosine_conv},
            col sep=comma,
        ] {wrn-sgd-wd-sched-large3_adv.csv};
        \addlegendentry{$L_2$, first layer}

        \addplot[color=tab4, mark=*, mark size=2pt, thick] table[
            x=epoch,
            y={l1_cosine_conv},
            col sep=comma,
        ] {wrn-sgd-wd-sched-large3_adv.csv};
        \addlegendentry{$L_1$, first layer}
        \end{axis}
\end{tikzpicture}
\caption{Alignement of linearized adversarial gap gradients with various regularizations for each layer along adversarial training trajectory (SGD, cosine schedule, $\eta_0 = 10^{-1}$ to $10^{-3}$, 100 epochs).}\label{fig:wrn-cosine}
\end{figure}

\section{Weakening and Bypassing Equivalence}\label{sec:conclusion}

We have shown that constructing a regularizers $\Omega_\delta(\mathcal{D}; \cdot)$ equivalent to adversarial training, in the sense that $\Omega_\delta(\mathcal{D}; \cdot) \sim_P \mathcal{A}_\delta(\mathcal{D}; \cdot)$, is inherently difficult.
In particular, for two layer neural networks, this equivalence cannot be achieved using data-independent regularizations or loss modulations, in contrast to the linear setting of \citet{ribeiro2023regularization}.
This leaves only data-dependent reformulations, such as \citet{roth2020adversarial}, which remain as expensive computationally as adversarial training itself.
Thus the search for cheaper alternatives to adversarial training cannot rely on a direct rewriting of the adversarial gap. We review two weakenings of this ideal property which could be fruitful directions.

\textbf{Weakening equivalence: local control around $(\mathcal{D}_0, \theta_0)$.}
The obstruction to data-independent surrogates is that
$A(\mathcal{D}; \cdot)$ depends non-trivially on the
distribution $\mathcal{D}$. In particular, the adversarial influence
on weights
$\partial_\theta A(\mathcal{D}; \cdot)$ changes direction when
$\mathcal{D}$ is altered. Either the definition of robustness can be weakened to
be more permissive \citep{pang2022robustness}, or a weaker equivalence can be used.
For instance, computing a single simplified equivalent $\Omega_0 \sim_P A(\mathcal{D}_0; \cdot)$, for some $\mathcal{D}_0$
(e.g. standard gaussian) the approximation could remain usable
around $\mathcal{D}_0$, such that $\Omega_0 \approx A(\mathcal{D}; \cdot)$ when $\mathcal{D} \approx \mathcal{D}_0$.

The idea of \textit{local control} has also been
gaining attention in the study of convergence of neural networks
by gradient descent.
Studying an idealized ``typical'' case and characterizing
a behavior allows to transfer this property to ``sufficiently typical''
instances.
For instance, the random initialization $\theta_0$ of a neural network
is well understood, especially for very wide networks, and the
trajectory around ``close enough'' parameters
$\lVert \theta - \theta_0 \rVert \leq R$ is shown to have
good optimization properties, allowing convergence proofs
\citep{chizat2018global,mei2018mean,pham2021global}
and later quantitative convergence speeds
\citep{scaman22a,robin2022convergence} leveraging this control.

\textbf{Bypassing equivalence: over-regularizing for robustness.}
Another option is to maintain a uniform claim over $\mathcal{D}$
by lowering its strength, such as
a regularization that is not equivalent to the adversarial gap,
but does enforce robustness without degrading performance too much. A natural option is the control of the global Lipschitz constant of the network, with estimates \citep{weng2018evaluating} or explicit bounds \citep{scaman2018lipschitz,piat2022towards}, since a sufficiently regular network is robust by construction \citep{luxburg2004distance,bartlett2017spectral,neyshabur2017exploring}.
The major obstacle is that there is no guarantee that learning such
a function is possible. Indeed, many regularization attempts have shown large
decreases in accuracy \citep{atsague2021mutual,
sriramanan2021towards,
wang2024regularization} and several optimization problems \citep{waseda2025rethinking}.

Theory-guided surrogates built with guarantees such as
${\mathcal{L}(\theta_+) \leq (1 + \varepsilon) \cdot \mathcal{L}(\theta^\star)}$
\citep[see][]{bartlett2008classification}
ensure the surrogate $\theta_+$ has manageable suboptimality gap
(even $\varepsilon \approx 1$).
With no such guarantees, it is unclear whether a surrogate would be usable
without \emph{extensive} experiments.
Adversarial training's ``minimal'' perturbation leading to robustness
justified the search for equivalents.

\newpage{}

\bibliography{bibliography}

\begin{thebibliography}{34}
\providecommand{\natexlab}[1]{#1}
\providecommand{\url}[1]{\texttt{#1}}
\expandafter\ifx\csname urlstyle\endcsname\relax
  \providecommand{\doi}[1]{doi: #1}\else
  \providecommand{\doi}{doi: \begingroup \urlstyle{rm}\Url}\fi

\bibitem[Atsague et~al.(2021)Atsague, Fakorede, and Tian]{atsague2021mutual}
Modeste Atsague, Olukorede Fakorede, and Jin Tian.
\newblock A mutual information regularization for adversarial training.
\newblock In Vineeth~N. Balasubramanian and Ivor Tsang, editors,
  \emph{Proceedings of The 13th Asian Conference on Machine Learning}, volume
  157 of \emph{Proceedings of Machine Learning Research}, pages 188--203. PMLR,
  17--19 Nov 2021.
\newblock URL \url{https://proceedings.mlr.press/v157/atsague21a.html}.

\bibitem[Bartlett and Wegkamp(2008)]{bartlett2008classification}
Peter~L. Bartlett and Marten~H. Wegkamp.
\newblock Classification with a reject option using a hinge loss.
\newblock \emph{Journal of Machine Learning Research}, 9\penalty0
  (59):\penalty0 1823--1840, 2008.
\newblock URL \url{http://jmlr.org/papers/v9/bartlett08a.html}.

\bibitem[Bartlett et~al.(2017)Bartlett, Foster, and
  Telgarsky]{bartlett2017spectral}
Peter~L. Bartlett, Dylan~J. Foster, and Matus Telgarsky.
\newblock Spectrally-normalized margin bounds for neural networks.
\newblock In \emph{Proceedings of the 31st International Conference on Neural
  Information Processing Systems}, NeurIPS'17, page 6241–6250, Red Hook, NY,
  USA, 2017. Curran Associates Inc.
\newblock ISBN 9781510860964.

\bibitem[Bartoldson et~al.(2024)Bartoldson, Diffenderfer, Parasyris, and
  Kailkhura]{bartoldson2024adversarial}
Brian~R. Bartoldson, James Diffenderfer, Konstantinos Parasyris, and Bhavya
  Kailkhura.
\newblock Adversarial robustness limits via scaling-law and human-alignment
  studies.
\newblock In \emph{Proceedings of the 41st International Conference on Machine
  Learning}, ICML'24. JMLR.org, 2024.

\bibitem[Biggio et~al.(2013)Biggio, Corona, Maiorca, Nelson, Srndic, Laskov,
  Giacinto, and Roli]{biggio2013evasion}
Battista Biggio, Igino Corona, Davide Maiorca, Blaine Nelson, Nedim Srndic,
  Pavel Laskov, Giorgio Giacinto, and Fabio Roli.
\newblock Evasion attacks against machine learning at test time.
\newblock In \emph{Joint European conference on machine learning and knowledge
  discovery in databases}, pages 387--402. Springer, 2013.

\bibitem[Blanchet et~al.(2019)Blanchet, Kang, and Murthy]{blanchet2019robust}
Jose Blanchet, Yang Kang, and Karthyek Murthy.
\newblock Robust wasserstein profile inference and applications to machine
  learning.
\newblock \emph{Journal of Applied Probability}, 56\penalty0 (3):\penalty0
  830–857, 2019.
\newblock \doi{10.1017/jpr.2019.49}.

\bibitem[Carlini et~al.(2024)Carlini, Nasr, Choquette-Choo, Jagielski, Gao,
  Koh, Ippolito, Tramer, and Schmidt]{carlini2024aligned}
Nicholas Carlini, Milad Nasr, Christopher~A Choquette-Choo, Matthew Jagielski,
  Irena Gao, Pang Wei~W Koh, Daphne Ippolito, Florian Tramer, and Ludwig
  Schmidt.
\newblock Are aligned neural networks adversarially aligned?
\newblock \emph{Advances in Neural Information Processing Systems}, 36, 2024.

\bibitem[Chizat and Bach(2018)]{chizat2018global}
L\'{e}na\"{\i}c Chizat and Francis Bach.
\newblock On the global convergence of gradient descent for over-parameterized
  models using optimal transport.
\newblock In S.~Bengio, H.~Wallach, H.~Larochelle, K.~Grauman, N.~Cesa-Bianchi,
  and R.~Garnett, editors, \emph{Advances in Neural Information Processing
  Systems}, volume~31. Curran Associates, Inc., 2018.
\newblock URL
  \url{https://proceedings.neurips.cc/paper_files/paper/2018/file/a1afc58c6ca9540d057299ec3016d726-Paper.pdf}.

\bibitem[Croce and Hein(2020)]{croce2020reliable}
Francesco Croce and Matthias Hein.
\newblock Reliable evaluation of adversarial robustness with an ensemble of
  diverse parameter-free attacks.
\newblock In \emph{Proceedings of the 37th International Conference on Machine
  Learning}, ICML'20. JMLR.org, 2020.

\bibitem[Goodfellow et~al.(2015)Goodfellow, Shlens, and
  Szegedy]{goodfellow2015explaining}
Ian~J. Goodfellow, Jonathon Shlens, and Christian Szegedy.
\newblock Explaining and harnessing adversarial examples.
\newblock In Yoshua Bengio and Yann LeCun, editors, \emph{3rd International
  Conference on Learning Representations, {ICLR} 2015, San Diego, CA, USA, May
  7-9, 2015, Conference Track Proceedings}, 2015.
\newblock URL \url{http://arxiv.org/abs/1412.6572}.

\bibitem[Krizhevsky and Hinton(2009)]{krizhevsky2009learning}
Alex Krizhevsky and Geoffrey Hinton.
\newblock Learning multiple layers of features from tiny images.
\newblock Technical Report~0, University of Toronto, Toronto, Ontario, 2009.
\newblock URL
  \url{https://www.cs.toronto.edu/~kriz/learning-features-2009-TR.pdf}.

\bibitem[Kumar et~al.(2019)Kumar, Brien, Albert, Viljöen, and
  Snover]{kumar2019failure}
Ram Shankar~Siva Kumar, David~O Brien, Kendra Albert, Salomé Viljöen, and
  Jeffrey Snover.
\newblock Failure modes in machine learning systems, 2019.
\newblock URL \url{https://arxiv.org/abs/1911.11034}.

\bibitem[Li et~al.(2019)Li, Fang, Xu, and Zhao]{li2019inductive}
Yan Li, Ethan~X. Fang, Huan Xu, and Tuo Zhao.
\newblock Inductive bias of gradient descent based adversarial training on
  separable data.
\newblock \emph{ArXiv}, abs/1906.02931, 2019.
\newblock URL \url{https://api.semanticscholar.org/CorpusID:174801103}.

\bibitem[Luxburg and Bousquet(2004)]{luxburg2004distance}
Ulrike~von Luxburg and Olivier Bousquet.
\newblock Distance--based classification with lipschitz functions.
\newblock \emph{J. Mach. Learn. Res.}, 5:\penalty0 669–695, December 2004.
\newblock ISSN 1532-4435.

\bibitem[Madry et~al.(2018)Madry, Makelov, Schmidt, Tsipras, and
  Vladu]{madry2018towards}
Aleksander Madry, Aleksandar Makelov, Ludwig Schmidt, Dimitris Tsipras, and
  Adrian Vladu.
\newblock Towards deep learning models resistant to adversarial attacks.
\newblock In \emph{6th International Conference on Learning Representations,
  {ICLR} 2018, Vancouver, BC, Canada, April 30 - May 3, 2018, Conference Track
  Proceedings}. OpenReview.net, 2018.
\newblock URL \url{https://openreview.net/forum?id=rJzIBfZAb}.

\bibitem[Mei et~al.(2018)Mei, Montanari, and Nguyen]{mei2018mean}
Song Mei, Andrea Montanari, and Phan-Minh Nguyen.
\newblock A mean field view of the landscape of two-layer neural networks.
\newblock \emph{Proceedings of the National Academy of Sciences}, 115\penalty0
  (33), July 2018.
\newblock ISSN 1091-6490.
\newblock \doi{10.1073/pnas.1806579115}.
\newblock URL \url{http://dx.doi.org/10.1073/pnas.1806579115}.

\bibitem[Neyshabur et~al.(2017)Neyshabur, Bhojanapalli, McAllester, and
  Srebro]{neyshabur2017exploring}
Behnam Neyshabur, Srinadh Bhojanapalli, David McAllester, and Nathan Srebro.
\newblock Exploring generalization in deep learning.
\newblock In \emph{Proceedings of the 31st International Conference on Neural
  Information Processing Systems}, NeurIPS'17, page 5949–5958, Red Hook, NY,
  USA, 2017. Curran Associates Inc.
\newblock ISBN 9781510860964.

\bibitem[Paleyes et~al.(2022)Paleyes, Urma, and
  Lawrence]{paleyes2022challenges}
Andrei Paleyes, Raoul-Gabriel Urma, and Neil~D. Lawrence.
\newblock Challenges in deploying machine learning: A survey of case studies.
\newblock \emph{ACM Comput. Surv.}, 55\penalty0 (6), December 2022.
\newblock ISSN 0360-0300.
\newblock \doi{10.1145/3533378}.
\newblock URL \url{https://doi.org/10.1145/3533378}.

\bibitem[Pang et~al.(2022)Pang, Lin, Yang, Zhu, and Yan]{pang2022robustness}
Tianyu Pang, Min Lin, Xiao Yang, Jun Zhu, and Shuicheng Yan.
\newblock Robustness and accuracy could be reconcilable by ({P}roper)
  definition.
\newblock In Kamalika Chaudhuri, Stefanie Jegelka, Le~Song, Csaba Szepesvari,
  Gang Niu, and Sivan Sabato, editors, \emph{Proceedings of the 39th
  International Conference on Machine Learning}, volume 162 of
  \emph{Proceedings of Machine Learning Research}, pages 17258--17277. PMLR,
  17--23 Jul 2022.
\newblock URL \url{https://proceedings.mlr.press/v162/pang22a.html}.

\bibitem[Pham and Nguyen(2021)]{pham2021global}
Huy~Tuan Pham and Phan-Minh Nguyen.
\newblock Global convergence of three-layer neural networks in the mean field
  regime.
\newblock In \emph{International Conference on Learning Representations}, 2021.
\newblock URL \url{https://openreview.net/forum?id=KvyxFqZS_D}.

\bibitem[Piat et~al.(2022)Piat, Fadili, Jurie, and da~Veiga]{piat2022towards}
William Piat, Jalal Fadili, Fr{\'e}d{\'e}ric Jurie, and S{\'e}bastien da~Veiga.
\newblock {Towards an Evaluation of Lipschitz Constant Estimation Algorithms by
  building Models with a Known Lipschitz Constant}.
\newblock In \emph{{Workshop on Trustworthy Artificial Intelligence as a part
  of the ECML/PKDD 22 program}}, Grenoble, France, France, September 2022. {IRT
  SystemX [IRT SystemX]}.
\newblock URL \url{https://hal.science/hal-03773372}.

\bibitem[Ribeiro et~al.(2023)Ribeiro, Zachariah, Bach, and
  Sch{\"o}n]{ribeiro2023regularization}
Antonio~H. Ribeiro, Dave Zachariah, Francis Bach, and Thomas~B. Sch{\"o}n.
\newblock Regularization properties of adversarially-trained linear regression.
\newblock In \emph{Thirty-seventh Conference on Neural Information Processing
  Systems}, 2023.
\newblock URL \url{https://openreview.net/forum?id=K8gLHZIgVW}.

\bibitem[Ribeiro et~al.(2025)Ribeiro, Sch{\"o}n, Zachariah, and
  Bach]{ribeiro2025efficient}
Antonio~H. Ribeiro, Thomas~B. Sch{\"o}n, Dave Zachariah, and Francis Bach.
\newblock Efficient optimization algorithms for linear adversarial training.
\newblock In Yingzhen Li, Stephan Mandt, Shipra Agrawal, and Emtiyaz Khan,
  editors, \emph{Proceedings of The 28th International Conference on Artificial
  Intelligence and Statistics}, volume 258 of \emph{Proceedings of Machine
  Learning Research}, pages 1207--1215. PMLR, 03--05 May 2025.
\newblock URL \url{https://proceedings.mlr.press/v258/ribeiro25a.html}.

\bibitem[Robin et~al.(2022)Robin, Scaman, and Lelarge]{robin2022convergence}
David A.~R. Robin, Kevin Scaman, and Marc Lelarge.
\newblock Convergence beyond the over-parameterized regime using rayleigh
  quotients.
\newblock In S.~Koyejo, S.~Mohamed, A.~Agarwal, D.~Belgrave, K.~Cho, and A.~Oh,
  editors, \emph{Advances in Neural Information Processing Systems}, volume~35,
  pages 10725--10736. Curran Associates, Inc., 2022.
\newblock URL
  \url{https://proceedings.neurips.cc/paper_files/paper/2022/file/45d74e190008c7bff2845ffc8e3facd3-Paper-Conference.pdf}.

\bibitem[Roth et~al.(2020)Roth, Kilcher, and Hofmann]{roth2020adversarial}
Kevin Roth, Yannic Kilcher, and Thomas Hofmann.
\newblock Adversarial training is a form of data-dependent operator norm
  regularization.
\newblock In H.~Larochelle, M.~Ranzato, R.~Hadsell, M.F. Balcan, and H.~Lin,
  editors, \emph{Advances in Neural Information Processing Systems}, volume~33,
  pages 14973--14985. Curran Associates, Inc., 2020.
\newblock URL
  \url{https://proceedings.neurips.cc/paper_files/paper/2020/file/ab7314887865c4265e896c6e209d1cd6-Paper.pdf}.

\bibitem[Scaman and Virmaux(2018)]{scaman2018lipschitz}
Kevin Scaman and Aladin Virmaux.
\newblock Lipschitz regularity of deep neural networks: analysis and efficient
  estimation.
\newblock In \emph{Proceedings of the 32nd International Conference on Neural
  Information Processing Systems}, NeurIPS'18, Red Hook, NY, USA, 2018. Curran
  Associates Inc.

\bibitem[Scaman et~al.(2022)Scaman, Malherbe, and Santos]{scaman22a}
Kevin Scaman, Cedric Malherbe, and Ludovic~Dos Santos.
\newblock Convergence rates of non-convex stochastic gradient descent under a
  generic lojasiewicz condition and local smoothness.
\newblock In Kamalika Chaudhuri, Stefanie Jegelka, Le~Song, Csaba Szepesvari,
  Gang Niu, and Sivan Sabato, editors, \emph{Proceedings of the 39th
  International Conference on Machine Learning}, volume 162 of
  \emph{Proceedings of Machine Learning Research}, pages 19310--19327. PMLR,
  17--23 Jul 2022.
\newblock URL \url{https://proceedings.mlr.press/v162/scaman22a.html}.

\bibitem[Sriramanan et~al.(2021)Sriramanan, Addepalli, Baburaj, and
  R]{sriramanan2021towards}
Gaurang Sriramanan, Sravanti Addepalli, Arya Baburaj, and Venkatesh~Babu R.
\newblock Towards efficient and effective adversarial training.
\newblock In M.~Ranzato, A.~Beygelzimer, Y.~Dauphin, P.S. Liang, and J.~Wortman
  Vaughan, editors, \emph{Advances in Neural Information Processing Systems},
  volume~34, pages 11821--11833. Curran Associates, Inc., 2021.
\newblock URL
  \url{https://proceedings.neurips.cc/paper_files/paper/2021/file/62889e73828c756c961c5a6d6c01a463-Paper.pdf}.

\bibitem[Szegedy et~al.(2014)Szegedy, Zaremba, Sutskever, Bruna, Erhan, and
  ad~Rob~Fergus]{Szegedy2013IntriguingPO}
Christian Szegedy, Wojciech Zaremba, Ilya Sutskever, Joan Bruna, Dumitru Erhan,
  and Ian~Goodfellow ad~Rob~Fergus.
\newblock Intriguing properties of neural networks.
\newblock In \emph{International Conference on Learning Representations}, 2014.

\bibitem[Wang et~al.(2024)Wang, Gao, and Xie]{wang2024regularization}
Jie Wang, Rui Gao, and Yao Xie.
\newblock Regularization for adversarial robust learning.
\newblock \emph{ArXiv}, abs/2408.09672, 2024.
\newblock URL \url{https://api.semanticscholar.org/CorpusID:271903109}.

\bibitem[Waseda et~al.(2025)Waseda, Chang, and Echizen]{waseda2025rethinking}
Futa~Kai Waseda, Ching-Chun Chang, and Isao Echizen.
\newblock Rethinking invariance regularization in adversarial training to
  improve robustness-accuracy trade-off.
\newblock In \emph{The Thirteenth International Conference on Learning
  Representations}, 2025.
\newblock URL \url{https://openreview.net/forum?id=M9SKazbVkJ}.

\bibitem[Weng et~al.(2018)Weng, Zhang, Chen, Yi, Su, Gao, Hsieh, and
  Daniel]{weng2018evaluating}
Tsui-Wei Weng, Huan Zhang, Pin-Yu Chen, Jinfeng Yi, Dong Su, Yupeng Gao,
  Cho-Jui Hsieh, and Luca Daniel.
\newblock Evaluating the robustness of neural networks: An extreme value theory
  approach.
\newblock In \emph{International Conference on Learning Representations}, 2018.
\newblock URL \url{https://openreview.net/forum?id=BkUHlMZ0b}.

\bibitem[Zagoruyko and Komodakis(2016)]{zagoruyko2016wide}
Sergey Zagoruyko and Nikos Komodakis.
\newblock Wide residual networks.
\newblock In \emph{British Machine Vision Conference 2016}. British Machine
  Vision Association, 2016.

\bibitem[Zhang et~al.(2019)Zhang, Yu, Jiao, Xing, Ghaoui, and
  Jordan]{zhang2019theoretically}
Hongyang Zhang, Yaodong Yu, Jiantao Jiao, Eric Xing, Laurent~El Ghaoui, and
  Michael Jordan.
\newblock Theoretically principled trade-off between robustness and accuracy.
\newblock In Kamalika Chaudhuri and Ruslan Salakhutdinov, editors,
  \emph{Proceedings of the 36th International Conference on Machine Learning},
  volume~97 of \emph{Proceedings of Machine Learning Research}, pages
  7472--7482. PMLR, 09--15 Jun 2019.
\newblock URL \url{https://proceedings.mlr.press/v97/zhang19p.html}.

\end{thebibliography}
\bibliographystyle{plainnat}

\newpage
\appendix
\onecolumn

\section{Proofs}

\subsection{Proof of good definition of $P$.}
\label{appendix:definition}

\begin{definition}
Define the set $P \subseteq \mathcal{F}^2 \times \mathbb{R}_+^*$ as
\[
P = \{\, (\Omega_1, \Omega_2, \lambda) \in \mathcal{F}^2 \times \mathbb{R}_+^* \mid
  \exists c \in \mathbb{R}, \> \Omega_2 = \lambda \cdot \Omega_1 + c \,\}
\]
$P$ induces the equivalence relation ${(\Omega_1 \sim_P \Omega_2) \Leftrightarrow (\exists \lambda \in \mathbb{R}_+^*, \, (\Omega_1, \Omega_2, \lambda) \in P)}$.
\end{definition}

\begin{proof}
Let us check that $\sim_P$ is an equivalence relation.
It is reflexive, because $(\Omega, \Omega, 1) \in P$;
%($\forall \Omega \in \mathcal{F}, \Omega \sim_P \Omega$, trivially),
symmetric, because
$(\Omega_1, \Omega_2, \lambda) \in P \Rightarrow (\Omega_2, \Omega_1, \lambda^{-1}) \in P$;
and transitive because if
${(\Omega_1, \Omega_2, \lambda_1) \in P}$ and ${(\Omega_2, \Omega_3, \lambda_2) \in P}$,
then there are constants $c_1, c_2$ such that $\Omega_2 = \lambda_1 \cdot \Omega_1 + c_1$
and $\Omega_3 = \lambda_2 \cdot \Omega_2 + c_2$, therefore
$\Omega_3 = (\lambda_2 \cdot \lambda_1) \cdot \Omega_1 + (\lambda_2 \cdot c_1 + c_2)$
hence $(\Omega_1, \Omega_3, \lambda_1\lambda_2) \in P$.
\end{proof}

\begin{proposition}
    If $(\Phi_n \in \mathcal{F})_{n \in \mathbb{N}}$ and $(\Psi_n \in \mathcal{F})_{n \in \mathbb{N}}$ converge to ${\Phi_n \to \Phi_\infty \in \mathcal{F}}$,
    ${\Psi_n \to \Psi_\infty \in \mathcal{F}}$, and if there exists $(\lambda_n \in \mathbb{R}_+^*)_{n \in \mathbb{N}}$
    such that $(\Phi_n, \Psi_n, \lambda_n) \in P$ and $\lambda \in \Theta(1)$,
    then $\Phi_\infty \sim_P \Psi_\infty$.
\end{proposition}

\begin{proof}
    Define $c : \mathbb{N} \to \mathbb{R}$ such that $\Psi_n = \lambda_n \Phi_n + c_n$.
    Proceed by case disjunction on $\Phi_\infty$.

    If $\Phi_\infty$ is not constant, then pick $(u,v) \in \Theta$
    such that $\Phi_\infty(u) \neq \Phi_\infty(v)$. After a certain rank, it thus holds that
    $\Phi_n(u) \neq \Phi_n(v)$, and thus ${\lambda_n = (\Psi_n(u) - \Psi_n(v)) / (\Phi_n(u) - \Phi_n(v))}$.
    In particular, by convergence of $(\Psi_n)_n$ and $(\Phi_n)_n$, this implies
    convergence of $(\lambda_n)_n$ to
    ${\lambda_n \to \lambda_\infty := (\Psi_\infty(u) - \Psi_\infty(v)) / (\Phi_\infty(u) - \Phi_\infty(v)) \in \mathbb{R}_+}$. Moreover, since $(\lambda_n)_n \in \Theta(1)$,
    this implies that $\lambda_\infty \in \mathbb{R}_+^*$.
    Therefore $c_n = \lambda_n \cdot \Phi_n(u) - \Psi_n(u) \to c_\infty := \lambda_\infty \cdot \Phi_\infty(u) - \Psi_\infty(u) \in \mathbb{R}$.
    Thus if $\Phi_\infty$ is not constant, then it holds $\Psi_\infty = \lambda_\infty \cdot \Phi_\infty + c_\infty$, thus $(\Phi_\infty, \Psi_\infty, \lambda_\infty) \in P$.

    On the other hand if $\Phi_\infty$ is constant, then let us show that $\Psi_\infty$ is
    constant as well, which will immediately imply $\Psi_\infty = \lambda \Phi_\infty + c$ for any $\lambda$.
    For any $(u,v) \in \Theta^2$,
    observe that ${\Psi_n(u) - \Psi_n(v) = \lambda_n \cdot (\Phi_n(u) - \Phi_n(v))}$, and
    that $(\lambda_n)_n \in \Theta(1)$
    implies that $(\lambda_n)_n$ is bounded above, thus
    ${\lvert \Psi_n(u) - \Psi_n(v) \rvert \leq (\sup \lambda_n) \cdot \lvert \Phi_n(u) - \Phi_n(v) \rvert \to 0}$,
    hence $\Psi_\infty(u) = \Psi_\infty(v)$.

    In both cases, $\Psi_\infty = \lambda_\infty \cdot \Phi_\infty + c_\infty$ for
    some $\lambda_\infty \in I$, thus $(\Phi_\infty, \Psi_\infty) \in P[I]$.
\end{proof}

\convergencepropB*

\begin{proof}
    Apply the previous proposition to $(\Omega_\delta / \delta, \mathcal{A}_\delta(\mathcal{D}; \cdot) / \delta, \lambda_\delta) \in P$, having $(\lambda_\delta)_{\delta} \in \Theta(1)$.
\end{proof}

\subsection{Proof of \Proposition{aligned-optima}: alignment of optima}
\label{appendix:aligned-optima-proof}

\alignedoptimaprop*

\begin{proof}
The proof is immediate by observing that when $\Omega_2 = \lambda \cdot \Omega_1 + c$
for $\lambda > 0$, it holds for all $(u,v) \in \Theta^2$ that
$\Omega_2(u) \leq \Omega_2(v) \Leftrightarrow \Omega_1(u) \leq \Omega_1(v)$.
In particular, their minima coincide.
\end{proof}

\subsection{Proof of \Proposition{cosine}: alignment of gradients}
\label{appendix:aligned-gradients-proof}

\alignedgradientsprop*

\begin{proof}
The proof is immediate after observing that
$\Omega_2 = \lambda \cdot \Omega_1 + c$ for $\lambda > 0$
implies ${\nabla \Omega_2 = \lambda \cdot \nabla \Omega_1}$ therefore
$\nabla \Omega_1 \cdot Q \cdot \nabla \Omega_2 = \lambda \cdot (\nabla \Omega_1 \cdot Q \cdot \nabla \Omega_1)$ and $\nabla \Omega_2 \cdot Q \cdot \nabla \Omega_2 = \lambda^2 \cdot (\nabla \Omega_1 \cdot Q \cdot \nabla \Omega_1)$.
\end{proof}

This property can be used to ignore some coordinates and simplify coordinates.
For any choice of coordinates $I \subseteq [n]$, using vectors $e_i = ( \mathds{1}_{i = j})_{j \in [n]} \in \mathbb{R}^n$, we can construct the projection to the subspace defined by $I$ as $Q = \sum_{i \in I} e_i \, e_i^T \in \mathbb{R}^{n \times n}$,
for which $\langle a, b \rangle_Q = \sum_{j \in I} a_j \, b_j$.

\subsection{Technical lemmas}

\begin{assumption}\label{assumption:smooth-relu}
    $\sigma : \mathbb{R} \to \mathbb{R}$ is a differentiable function such that there exists
    $\varepsilon_0 > 0$ satisfying
    \[ \forall x \in \mathbb{R}, \quad \lvert x \rvert > \varepsilon_0 \Rightarrow \sigma(x) = \max(0, x) \]
\end{assumption}

This assumption simply states that $\sigma$ is \textit{essentially} a ReLU non-linearity, except
on a (non-observable) small neighborhood of zero, ensuring that all gradients used in adversarial training
are well-defined.

\vspace{.3cm}
\begin{lemma}[Differentiability]\label{lem:network-differentiable}
Let $\Theta = \mathbb{R}^{m \times d} \times \mathbb{R}^{m \times K}$,
and let $\sigma : \mathbb{R} \to \mathbb{R}$ be a non-linearity
satisfying Assumption~\ref{assumption:smooth-relu}.
The function $f : \Theta \times \mathbb{R}^d \to \mathbb{R}^K$ defined as
$$ f((w,a), x) = \sum_{i \in [m]} a_i \, \sigma(w_i \cdot x) $$
is continuous and differentiable,
and its derivative with respect to $x$ is
$$ \partial_x f((w,a), x) = \sum_{i \in [m]} a_i \otimes w_i \, \nabla\sigma(w_i \cdot x)
\in \mathbb{R}^{K \times d}$$
\end{lemma}
\begin{proof}
Continuity and differentiability are immediate as a consequence of differentiability of $\sigma$.
\end{proof}

\vspace{.3cm}
\begin{lemma}[Neuron isolation]\label{lem:computation-lemma}
Let $\Theta = \mathbb{R}^{m \times d} \times \mathbb{R}^{m \times K}$,
let $\sigma : \mathbb{R} \to \mathbb{R}$ be a non-linearity satisfying Assumption~\ref{assumption:smooth-relu}.
Define $f : \Theta \times \mathbb{R}^d \to \mathbb{R}^K$ the function
$ f((w,a), x) = \sum_{i \in [m]} a_i \, \sigma(w_i \cdot x) $.

If $d \geq q + 1$, and $m \geq q + 1$, then for any set
$(w_i \in \mathbb{R}^d)_{i \in [q]}$ of orthonormal vectors,
i.e.\ such that $w_i \cdot w_j = \mathds{1}_{i=j}$,
there exists $u \in \mathbb{R}^{m \times d}$
and $(x_i \in \mathbb{R}^d)_{i \in [q]}$
such that $\forall i \in [q], u_i = w_i$ and
such that for any $a \in \mathbb{R}^{m \times K}$
there is an open set $\mathcal{U} \subseteq \Theta$
with $(u, a) \in \mathcal{U}$ and verifying
for all $(u', a') \in \mathcal{U}$ that

\[
\forall (j,k) \in [m] \times [q], \quad
(j \neq k \Rightarrow \sigma(u'_j \cdot x_k) = 0)
\]

\[
\forall k \in [q], \,\quad \partial_x f((u',a'), x_k) = a'_k \otimes u_k'
\]
\end{lemma}

\begin{proof}
Using $d \geq q +1$, complete the orthonormal set $(w_0, \ldots, w_{q-1}) \in \mathbb{R}^{q \times d}$
into an orthonormal set $(w_0, \ldots, w_q) \in \mathbb{R}^{(q+1) \times d}$
by selecting a unit vector $w_q$ orthogonal to all others.
Since $m \geq q+1$, complete $(w_i)_{i \in [q]} \in \mathbb{R}^{q \times d}$ into
${u = (w_0, \ldots, w_{q-1}, w_q, w_q, \ldots, w_q) \in \mathbb{R}^{m \times d}}$.
Then define $x_k = (1 + \varepsilon_0) \cdot \bigl(w_k - \sum_{i \in [q+1], i \neq k} w_i \bigr)$.
Let us show that for all $i \in [q]$, $u_i \cdot x_i > +\varepsilon_0$
and $\forall j \in [m] \setminus \{i\}$, $u_j \cdot x_i < -\varepsilon_0$.
From there, it will follow that $\sigma(u_j \cdot x_i) = 0$ when $i \neq j$,
and therefore $\partial_x f((u,a), x_k) = a_k \otimes u_k$. It will only remain to note
that since the previous inequalities are strict, they can be extended to an open set around $u$.

For the first statement, let $i \in [q]$
and observe that $
w_i \cdot x_i = (1+\varepsilon_0) \, (w_i \cdot w_i - \sum_{j \in [q+1], j \neq i} w_i \cdot w_j) = (1+ \varepsilon_0)\, \lVert w_i \rVert_2^2 > \varepsilon_0
$.
For the second, proceed by case disjunction. If $j \in [q]$ and $j \neq i$, then
\[\begin{aligned}
w_i \cdot x_j &= (1 +\varepsilon_0) \,\Bigl(w_i \cdot w_j - \sum_{k \in [q+1], \,k \neq j} w_i \cdot w_k \Bigr)
\\ &= (1+\varepsilon_0) \,\Bigl( - w_i \cdot w_i + w_i \cdot w_j - \sum_{k \in [q+1], \,k \notin \{i,j\}} w_i \cdot w_k \Bigr)
\\ &= - (1+\varepsilon_0) \, \lVert w_i \rVert_2^2 < - \varepsilon_0
\end{aligned}\]

On the other hand if $j \geq q$ and $j \neq i$, then
\[
w_i \cdot x_j = w_q \cdot x_j = (1+\varepsilon_0) \,\Bigl( w_q \cdot w_j - \sum_{k \in [q+1], k \neq j} w_q \cdot w_k \Bigr)
= - (1+\varepsilon_0)\lVert w_q \rVert_2^2 < -\varepsilon_0
\]
To check that these properties remain true in an open set,
observe that
we can define for all $i \in [m]$ and $k \in [q]$ the
continuous function $g_{i,k} : (w',a') \mapsto B_{i,k} \cdot w_i' \cdot x_k$
where $B_{i,k} \in \{ \pm 1 \}$ is $B_{i,k} = 1$ if $i = k$ and $B_{i,k} = -1$ otherwise.

Observe that $g_{i,k}(w,a) > \varepsilon_0$ for all $(i,k)$ as shown above.
Let ${\xi := \frac{1}{2} (\min_{i,k} g_{i,k}(w,a)} -\varepsilon_0) > 0$,
and observe that
the preimage ${\mathcal{V}_{i,k} = g_{i,k}^{-1} \left( \, \interval[open]{g_{i,k}(w,a) - \xi}{ g_{i,k}(w,a) + \xi } \, \right)  \subsetneq \Theta}$ is an open set with
$g_{i,k}(\theta) > \varepsilon_0$ for all $\theta \in \mathcal{V}_{i,k}$.
In particular the intersection
$\mathcal{U} = \cap_{i,k} \mathcal{V}_{i,k} \subsetneq \Theta$
is an open set (as intersection of finitely many open sets) such that $g_{i,k}(\theta) > \varepsilon_0$
for all $(i,k)$ and all $\theta \in \mathcal{U}$.

This has two consequences. First for all $(u', a') \in \mathcal{U}$ and $k \in [q]$, it holds that
$u'_k \cdot x_k > + \varepsilon_0$ therefore $\sigma(u'_k \cdot x_k) > 0$ and $\nabla \sigma(u'_k \cdot x_k) = 1$,
and on the other hand for any $i \in [m] \setminus \{k\}$ it holds $u'_i \cdot x_k < - \varepsilon$ therefore
$\sigma(u'_i \cdot x_k) = 0$, which concludes the first claim.
Moreover, by direct computation,
$\partial_x f((u', a'),x_k) = \sum_{i \in [m]} a'_i \otimes u'_i \, \nabla \sigma(u'_i \cdot x_k) = a'_k \otimes u'_k$
which concludes the second claim and thus completes the proof.
\end{proof}

\begin{lemma}[Derivatives of homogeneous functions]\label{lem:homogeneous-trick}
    For $h \in \mathbb{N}$, if $\psi : \mathbb{R}^h \to \mathbb{R}$
    is a 1-homogeneous function, that is to say if it holds
    $\>\forall (\lambda, u) \in \mathbb{R}_+ \times \mathbb{R}^h,  \, \psi(\lambda u) = \lambda \psi(u)$,
    \linebreak[2]
    then for all $u \in \mathbb{R}^h$ such that $\psi$ is differentiable at $u$,
    it holds $u \cdot \nabla \psi(u) = \psi(u)$.
\end{lemma}

\begin{proof}
    Let $u \in \mathbb{R}^h$ such that $\psi$ is differentiable at $u$.
    For any $\lambda \in \mathbb{R}_+^*$, $\psi$ is also differentiable at $(\lambda u) \in \mathbb{R}^h$ since $\psi(\lambda u) = \lambda \psi(u)$, with $\lambda \nabla \psi(\lambda u) = \lambda \nabla \psi(u)$.
    It remains to differentiate the equality $\psi(\lambda u) = \lambda \,\psi(u)$ with respect
    to $\lambda$, to get $u \cdot \nabla \psi(\lambda u) = \psi(u)$.
    Evaluating at $\lambda = 1$ concludes.
\end{proof}

\begin{lemma}[Derivatives of cross entropy]\label{lem:ce-derivatives}
The cross-entropy function $\ell : [K] \times \mathbb{R}^K \to \mathbb{R}$ defined as
$\ell(y, z) = - z_y + \log (\sum_{i \in [k]} \exp(z_i))$
has the following derivatives
\[
\frac{\partial \ell}{\partial z_i}(y, z) = - \mathds{1}_{i = y} + \frac{ \exp(z_i) }{ \sum_{l \in [K]} \exp(z_l) }
\]
\[
\frac{\partial^2 \ell}{\partial z_i \partial z_j}(y, z) = \mathds{1}_{i = j} \frac{\exp(z_i)}{\sum_{l \in [K]} \exp(z_l)} - \frac{\exp(z_i)}{\sum_{l \in [K]} \exp(z_l)} \cdot \frac{ \exp(z_j)}{\sum_{l \in [K]} \exp(z_l)}
\]
A common way to write them more succinctly is
$\partial_z \ell(y, z) = p - \mathds{1}_y$ and
$\partial_z^2 \ell(y,z) = \operatorname{diag}(p) - p p^T$,
where $p \in \Delta_K \subseteq \mathbb{R}_+^K$ is the softargmax\footnote{%
This is often also called a \textit{softmax} in deep learning for historical reasons,
despite not approximating $\max$.
}
of $z$ defined as ${p_i = \exp(z_i) / \sum_{l \in [K]} \exp(z_l)}$.
\end{lemma}

\begin{lemma}\label{lem:nontriviality-ce-grad}
    For any $y \in [K]$ and $\gamma \in \mathbb{R}_+^*$, there exists
    $a \in \mathbb{R}^K$ such that $\partial_z \ell(y, \gamma\, a) \cdot a \neq 0$.
\end{lemma}

\begin{proof}
Without loss of generality we can show $\partial_z \ell(y, \gamma\, a) \cdot (\gamma \, a) \neq 0$
and absorb $\gamma$ into $a$, therefore it suffices to show the result for $\gamma = 1$.
Define $a = (\mathds{1}_{i \neq y})_{i \in [K]} \in \mathbb{R}^K$.
It holds by \Lemma{ce-derivatives} that
$\partial_z \ell(y, a) = p - \mathds{1}_y$ where $p \in \mathbb{R}_+^K$ is defined as
$p_i = \exp(a_i) / \sum_{j \in [K]} \exp(a_j)$.
Therefore
$$ \partial_z \ell(y, a) \cdot a = \sum_{i \in [K]} \frac{\exp(a_i)}{ \sum_{j \in [K]} \exp(a_j) } \, a_i  - a_y= \frac{(K-1) \, e^1}{1 + (K-1) \, e^1} - 0\neq 0 $$
\end{proof}

\newpage

\subsection{Proof of \Proposition{no-separability} : Rejection of separability}
\label{appendix:no-separability-proof}

\propnosep*

The core idea of the proof is that the set of minima of a separable function
has a cartesian product structure in parameter space. Since adversarial gap is uniquely
defined by the function $f(\theta, \cdot)$ and not by its parameters, it cannot have
such a product structure, because we can construct ``diagonals'' in parameter space
keeping constant the function $f(\theta, \cdot)$. This symmetry distinguishes them.
\vspace{-.2cm}

\begin{proof}[Proof]
    Proceeding by contradiction, assume that $\Omega_{+} \not\sim_P A(\mathcal{D}; \cdot)$.
    Pick ${w_0 \in \mathbb{R}^d \setminus \{ 0 \}}$,
    and define $u = (w_0, w_0, -\frac{1}{2} w_0, \ldots, - \frac{1}{2} w_0) \in \mathbb{R}^{m \times k}$.
    Then, define ${\mathcal{V} = \mathbb{R}^K \times \mathbb{R}^K \times \{ 0 \}^{(m-2) \times K} \subseteq \mathbb{R}^{m\times K} }$
    and $\mathcal{U} = \{ u \} \times \mathcal{V}$.
    Let $x = w_0$ and $y \in \mathcal{Y}$.
    By \Lemma{computation-lemma}, for any $\theta = (u,a) \in \mathcal{U}$,
    it holds ${f(\theta, x) = (a_0 + a_1)\, \sigma(w_0 \cdot x)}$
    and
    $ {\partial_x f(\theta, x) = \sum_{i \in [m]} a_i \otimes w_i\, \nabla \sigma(w_i \cdot x) }
    = {a_0 \otimes w_0 + a_1 \otimes w_0 }$.
    Therefore, for any linear weight ${a = (a_0, a_1, 0, \ldots, 0) \in \mathcal{V}}$,
    on the single-point dataset $\mathcal{D} = \{ (x, y) \}$,
    the linearized adversarial risk is
    $A(\mathcal{D};(u,a)) = \lVert \partial_z \ell(y, f((u,a), x)) \cdot
    (a_0 \otimes w_0 + a_1 \otimes w_0) \rVert_*$.
    \linebreak[2]
    When $a_1 = - a_0$, this reduces to
    $A(\mathcal{D}; (u,a)) = 0$,
    thus $(u, (a_0, -a_0, 0 \ldots)) \in \argmin_{\,\mathcal{U}} A(\mathcal{D}; \cdot)$.
    However by \Proposition{aligned-optima},
    $\argmin_{\,\mathcal{U}} A(\mathcal{D}; \cdot) = \argmin_{\,\mathcal{U}} \Omega_+$
    and furthermore by separability of the candidate $\Omega_+$, it holds that
    $\argmin_{\,\mathcal{U}} \Omega_+
    = \{ u \} \times \left( \argmin \varphi_0 \times \argmin \varphi_1 \times \{0\}^{(m-2) \times K} \right)$.
    Hence $a_0 \in \argmin \varphi_0$,
    and $(-a_0) \in \argmin \varphi_1$
    and since this is valid for all choices $a_0 \in \mathbb{R}^K$,
    it follows that $(\argmin \varphi_0) = \mathbb{R}^K$,
    and $(\argmin \varphi_1) = \mathbb{R}^K$,
    so both $\varphi_0$ and $\varphi_1$ are constant.

    Finally, let $a_0 \in \mathbb{R}^K$ be such that
    $\partial_z \ell(y, a_0 \, \sigma(w_0 \cdot x)) \cdot a_0 \neq 0$
    (a construction is given in \Lemma{nontriviality-ce-grad}.
    Define $\theta = (u, (a_0, 0, \ldots, 0)) \in \mathcal{U} \subset \Theta$,
    and observe that $f(\theta, x) = a_0 \, \sigma(w_0 \cdot x)$
    and $\partial_x f(\theta, x) = a_0 \otimes w_0$.
    Since by definition $A(\mathcal{D}; \theta) = \lVert \partial_z \ell(y, f(\theta, x)) \cdot \partial_x f(\theta,x) \rVert_*$,
    we get ${A(\mathcal{D}; \theta) = \lvert \partial_z \ell(y, a_0 \, \sigma(w_0 \cdot x)) \cdot a_0 \rvert \cdot \lVert w_0 \rVert_* \neq 0 }$
    hence $\theta \notin \argmin_\mathcal{U} A(\mathcal{D}; \cdot)$.
    Since this choice satisfies ${\theta \in \{u\} \times ( \argmin \varphi_0 \times \argmin \varphi_1 \times \{0\}^{(m-2) \times K}) = \argmin_\mathcal{U} \Omega_+}$, we have shown
    $\argmin_\mathcal{U} \Omega_+ \neq \argmin_\mathcal{U} A(\mathcal{D}; \cdot)$
    which contradicts \Proposition{aligned-optima} and concludes the proof.
\end{proof}

\subsection{Proof of \Proposition{no-data-independence} : Rejection of data-independence}
\label{appendix:no-data-independence-proof}

\propnodataindependence*

We will show that there exists datasets $\mathcal{D}_0$ and $\mathcal{D}_1$
such that $A(\mathcal{D}_0; \cdot) \not\sim_P A(\mathcal{D}_1; \cdot)$.
In particular, this implies by transitivity that there exists $k \in \{0,1\}$
such that $\Omega_2 \not\sim_P A(\mathcal{D}_k, \cdot)$.

The core idea is that neurons can become ``deactivated''
(that is to say they have a locally constant response)
which will ``prevent'' gradients of $A(\mathcal{D}, \cdot)$ from ``flowing through them'',
therefore if there are two points activating distinct sets of neurons,
the weighted averaged of gradients induced by these points changes direction
(and even support) under variations of $\mathcal{D}$.
In other words, $A(\mathcal{D}; \cdot)$ depends non-trivially
on the dataset and thus cannot be equivalent to a data-independent function.

\begin{proof}[Proof]
Let $(w_0, w_1) \in \mathbb{R}^d \times \mathbb{R}^d$ be orthonormal vectors.
By \Lemma{computation-lemma}, there exists points $(x_0, x_1) \in \mathbb{R}^d \times \mathbb{R}^d$ and
${u \in \mathbb{R}^{m \times d}}$
with $u_0 = w_0$ and $u_1 = w_1$
such that for all $a \in \mathbb{R}^{m \times K}$,
and on an open set $\mathcal{U} \subseteq \Theta$ around $(u,a) \in \mathcal{U}$, it holds
$\partial_x f((w', a'), x_k) = a'_k \otimes w'_k$ for all $k \in \{0,1\}$
and $\sigma(w_j \cdot x_k) = 0$ if $j \neq k$.
Then, let $(y_0, y_1) \in \mathcal{Y}^2$, and define
for any $\alpha \in \interval{0}{1}$ the distribution $\mathcal{D}_\alpha$
composed of two points:
$(x_0, y_0)$ with weight $\alpha_0 = (1 - \alpha)$ and $(x_1, y_1)$ with
weight $\alpha_1 = \alpha$.

To shorten notations for derivatives of
$\mathcal{L}(\theta, x, y) = \ell(y, f(\theta, x))$,
define the gradient shorthand
${r_k : a \mapsto \partial_u \ell(y_k, f((u,a), x_k)) \in \mathbb{R}^k}$.
Observe then that by chain rule,
$ \partial_x \mathcal{L}(\theta, y_k, x_k) = \partial_z \ell(y_k, f(\theta, x_k)) \cdot \partial_x f(\theta, x_k)$
thus using for $\theta = (u,a)$ the explicit expression $\partial_x f(\theta, x_k) = a_k \otimes w_k$ and the previous shorthand, we get
$\partial_x \mathcal{L}((u,a), y_k, x_k) = (r_k(a) \cdot a_k) \, w_k \in \mathbb{R}^d$.
In particular using the definition of $A(\mathcal{D}; \cdot)$,
it holds that
$A(\mathcal{D}_\alpha; (u,a))
= \sum_{k \in \{0,1\}} \alpha_k \, \lVert ( r_k(a) \cdot a_k) \, w_k \lVert_*$.

Moreover, $A(\mathcal{D}_\alpha; \cdot)$ is differentiable
almost everywhere on $\mathcal{U}$.
Thus the expression of the gradient used below is well-defined almost
everywhere on $\mathcal{U}$, and it suffices to show with this expression that
$\lim_{\theta \to (w,a)} \, \cos_P\left( \partial_\theta A(\mathcal{D}_0; \theta), \partial_\theta A(\mathcal{D}_1; \theta) \right)  < 1$, such
that there is at least one $\theta' \in \mathcal{U}$ for which
$\cos_P \left( \partial_\theta A(\mathcal{D}_0; \theta'), \partial_\theta A(\mathcal{D}_1; \theta') \right) < 1$, and the conclusion follows immediately.

By construction of $u$, it holds for any $k \in \{0,1\}$ that
$f(\theta, x_k) = a_k \, \sigma(w_k \cdot x_k)$, by canceling the other terms
using $\sigma(w_j \cdot x_k) = 0$ when $j \neq k$.
This implies that $r_k(a) = \partial_z \ell(y, a_k \,\sigma(w_k \cdot x_k))$.
In particular, this means $\partial_{a_j} r_k(a) = 0$ for all $j \in [m] \setminus \{k \}$.
By the following straightforward computation, we show that for $j \in [m] \setminus \{k\}$,
it also holds $\partial_{a_j} \left( r_k(a) \cdot a_k \right) = 0$, since
by chain rule
$$
\frac{\partial}{\partial a_j} (r_k(a) \cdot a_k)
= \frac{\partial r_k}{\partial a_j}(a) \cdot a_k + r_k(a) \, \mathds{1}_{j=k}
$$

This can be leveraged to simplify terms in the gradient of $A(\mathcal{D}_\alpha; \cdot)$.
$$\begin{aligned}
\frac{\partial A}{\partial a_j}(\mathcal{D}_\alpha; (u,a))
&= \sum_{k \in \{0,1\}} \alpha_k \frac{\partial}{\partial a_j} \Bigl( \lvert r_k(a) \cdot a_k\rvert \cdot \lVert w_k \rVert_* \Bigr)
\\ &= \sum_{k \in \{0,1\}} \alpha_k
\, \lVert w_k \rVert_*
\, \operatorname{sgn}(r_k(a) \cdot a_k)
\, \frac{\partial \left( r_k(a) \cdot a_k\right)}{\partial a_j}
\\ &= \alpha_j \lVert w_j \rVert_* \operatorname{sgn}(r_j(a) \cdot a_j)
\, \frac{\partial \left( r_j(a) \cdot a_j\right)}{\partial a_j}
\end{aligned}$$

In particular, this implies that for any $a \in \mathbb{R}^{m \times K}$ it holds
$\partial A(\mathcal{D}_1; (u,a)) / \partial a_0 = 0$
which corresponds to $\alpha = (0, 1)$
and
$\partial A(\mathcal{D}_0; (u,a)) / \partial a_1 = 0$
corresponding to $\alpha = (1, 0)$.
To conclude, use $Q$ the PSD form
$Q((u,a), (u', a')) := \sum_{k \in \{0,1\}} \sum_{i \in [K]} a_{k,i} a'_{k,i}$,
which ensures by definition that
\[
\pprod{Q}{\partial_\theta A(\mathcal{D}_0; \theta)}{\partial_\theta A(\mathcal{D}_1; \theta)}
=
\dfrac{\partial A(\mathcal{D}_0; \theta)}{\partial a_0} \cdot
\dfrac{\partial A(\mathcal{D}_1; \theta)}{\partial a_0}
+
\dfrac{\partial A(\mathcal{D}_0; \theta)}{\partial a_1} \cdot
\dfrac{\partial A(\mathcal{D}_1; \theta)}{\partial a_1}
\]
Thus
$\pprod{Q}{\partial_\theta A(\mathcal{D}_0; \theta)}{\partial_\theta A(\mathcal{D}_1; \theta)} = 0$
for all $\theta = (u,a)$.
It only remains to choose $a \in \mathbb{R}^{m \times K}$ such that both
${\pnorm{Q}{\partial_\theta A(\mathcal{D}_0; \theta)} \neq 0}$
and ${\pnorm{Q}{\partial_\theta A(\mathcal{D}_1; \theta)} \neq 0}$,
which will constitute a violation of \Proposition{cosine} and conclude the proof.
Let us construct such a choice of $a$ explicitly.

Let $v = (\mathds{1}_{i\neq y})_{i \in [K]} \in \mathbb{R}^K$, then
choose $a_0 = v / \sigma(w_0 \cdot x_0)$ and $a_1 = v / \sigma(w_1 \cdot x_1)$.
Observe that for any $k \in \{0,1\}$, it holds
$r_k(a) = \partial_z \ell(y, a_k \, \sigma(w_k \cdot x_k))$
thus
$${\partial_{a_k} r_k(a) = \partial_z^2 \ell(y, a_k \, \sigma( w_k \cdot x_k)) \, \sigma(w_k \cdot x_k) = \partial_z^2 \ell(y, v) \, \sigma(w_k \cdot x_k) \in \mathbb{R}^{K \times K}}$$
Therefore using $p := \operatorname{softargmax}(v) \in \mathbb{R}^K_+$ in the expression derived by chain rule above,
together with the closed-form expression for the derivatives of the cross entropy (see \Lemma{ce-derivatives}), we get
$$\begin{aligned}
\xi := \frac{\partial}{\partial a_k} \bigl( r_k(a) \cdot a_k \bigr)
= \partial_z^2 \ell(y, v) \cdot v + \partial_z \ell(y, v)
= \left( p_i v_i - p_i \sum_{j \in [K]} p_j v_j  + p_i - \mathds{1}_{i = y}\right)_{i \in [K]}
\end{aligned}$$
In particular evaluating $v_i = \mathds{1}_{i \neq y}$, we get
$\xi_y = - p_y (1 - p_y) + p_y - 1 = p_y^2 - 1 < 0$
because $p_y < 1$ since $p$ is in the image of $\operatorname{softargmax}$.
Hence $\xi \neq 0$, thus $ \lVert \partial_{a_k} A(\mathcal{D}_k; \theta) \rVert_2^2 = \lVert \xi \rVert_2^2 \neq 0$ and for any $k \in \{0,1\}$
it follows that $\lVert \partial_\theta A(\mathcal{D}_k; \theta) \rVert_Q^2
= \lVert \partial_{a_k} A(\mathcal{D}_k ; \theta) \rVert_2^2 \neq 0$
which concludes the proof.
\end{proof}

\subsection{Proof of \Proposition{no-modulation} : Rejection of loss-modulation}
\label{appendix:no-modulation-proof}

\propnomodulation*

    The core idea for this proof is a combination of the previous two ideas:
$\partial_\theta A(\mathcal{D}; \cdot)$ moves non-trivially with $\mathcal{D}$,
in a way that depends on parameters only through the learned function $f(\theta, \cdot)$.
The loss-modulated form $\Omega_+(\mathcal{D}; \cdot)$ also depends non-trivially
on $\mathcal{D}$, but the ``shape'' of this dependence is somewhat ``separable'' with respect
to $(\mathcal{D}, \theta)$, thus it cannot remain aligned with $A(\mathcal{D}; \cdot)$
as $\mathcal{D}$ varies.

\begin{proof}
    Let us proceed by contradiction.
    Assume that $\Omega_+(\mathcal{D}; \cdot) \sim_P A(\mathcal{D}; \cdot)$ for all distributions $\mathcal{D}$,
    and let us construct a data distribution for which the equality of \Proposition{cosine} is violated.

    Let $(w_0, w_1)$ be two orthonormal vectors in $\mathbb{R}^d$.
    By \Lemma{computation-lemma}, since $m \geq 3$ and $d \geq 3$,
    there exists $u \in \mathbb{R}^{m \times d}$ and $(x_k \in \mathbb{R}^d)_{k \in \{0,1\}}$
    such that $u_0 = w_0$ and $u_1 = w_1$, and such that for any $a \in \mathbb{R}^{m \times k}$ there exists an open set
    $\mathcal{U} \subseteq \Theta$ around $(u,a) \in \Theta$ such that for all $(u', a') \in \mathcal{U}$ it holds
    for all $j \in [m]$ and $k \in \{0,1\}$ that if $j \neq k$ then $\sigma(u'_j \cdot x_k) = 0$.
    For shortness in the following, we use for $s \in \{0,1\}$ the shorthand
    $\sigma_s := \sigma(w_s \cdot x_s) \in \mathbb{R}_+^*$.

    Let us choose $y \in [K]$ in a way that will simplify later technicalities.
    For any $\omega \in [K]$, define $c_\omega := (1/K) \sum_{i \in [K]} H_i(u,0) - H_\omega(u,0) \in \mathbb{R}^h$. Since $\sum_{\omega \in [K]} c_\omega / K = 0$,
    by convexity of $\phi$ it holds $\sum_{\omega \in [K]} \phi(c_\omega) / K \geq \phi( \sum_{\omega \in [K]} c_\omega / K) = \phi(0)$.
    In particular, there must exist $\omega \in [K]$ such that $\phi(c_\omega) \geq \phi(0)$.
    Define $y \in [K]$ to be such that $\phi(c_y) = \max_{\omega \in [K]} \phi(c_\omega) \geq \phi(0)$.

    Define the data distributions $\mathcal{D}_0 = \{ (x_0, y) \}$ and $\mathcal{D}_1 = \{ (x_1, y) \}$.
    By the equivalence assumption, there exists $(\lambda_s \in \mathbb{R}_+^*)_{s \in \{0,1\}}$ such that
    for all $s \in \{0,1\}$, it holds
    $\partial_\theta \Omega_+(\mathcal{D}_s; \cdot) = \lambda_s \, \partial_\theta A(\mathcal{D}_s; \cdot)$.

    Define $r_s : a \mapsto \partial_z \ell(y, f((u,a), x_s)) \in \mathbb{R}^K$.
    Observe ${f((u,a), x_s) = a_s \sigma(u_s \cdot x_s) = a_s \sigma_s \in \mathbb{R}^K}$,
    and thus $\partial_{a_j} r_s(a) = 0$ when $j \neq s$.
    Additionally, note that $\partial_{a_s} r_s(a) = \partial_z^2 \ell(y, a_s \sigma_s) \, \sigma_s \in \mathbb{R}^{K \times K}$.
    \linebreak[2]
    Observe that by definition $\Omega_+(\mathcal{D}_s, (u,a)) =\phi(r_s(a) \cdot H(u,a))$
    and by definition of the adversarial gap
    ${ A(\mathcal{D}_s, (u,a)) = \psi( r_s(a) \cdot \partial_x f((u,a), x_s)) = \psi(r_s(a) \cdot (a_s \otimes w_s)) }$.
    These expressions enable the computation of the follow derivatives for any $j \in [m]$:
    \begin{equation}\label{eq:short-derivative-computation}\tag{E1}\begin{aligned}
    \frac{ \partial \Omega_+ }{\partial a_j} (\mathcal{D}_s; (u,a))
    = \Bigl( \partial_{a_j} r_s(a) \cdot H(u,a) + r_s(a) \cdot \partial_{a_j} H(u,a) \Bigr)
    &\cdot \nabla \phi\bigl(r_s(a) \cdot H(u,a)\bigr)
    \\
    \frac{\partial A}{\partial a_j} (\mathcal{D}_s; (u,a))
    = \mathds{1}_{j=s} \, \Bigl(\partial_{a_s} r_s(a) \cdot (a_s \otimes w_s) + r_s(a) \otimes w_s \Bigr)
    &\cdot \nabla \psi\bigl(r_s(a) \cdot (a_s \otimes w_s)\bigr)
    \end{aligned}\end{equation}

    We are now ready to construct a point violating the alignment equality.
    Let ${v := (\mathds{1}_{i \neq y})_{i \in [K]} \in \mathbb{R}^K}$,
    and $\gamma := \log(1/2) \in \mathbb{R}$.
    Then, let ${a^0 := (\gamma \, v / \sigma_0, \gamma \, v / \sigma_1, 0, \ldots, 0) \in \mathbb{R}^{m \times K}}$
    and $\theta := (u, a^0) \in \Theta$.

    Write $p := \operatorname{softargmax}(\gamma \, v) \in \mathbb{R}_+^K$.
    Recall that $r_s(a^0) = \partial_u \ell(y, a^0_s \sigma_s) = \partial_u \ell(y, \gamma \, v)$
    therefore $r_s(a^0) = p - \mathds{1}_y$.
    Moreover, note that
    $\partial_{a_s} r_s(a^0) = \partial_u^2 \ell(y, a^0_s\sigma_s)\, \sigma_s
    = (\operatorname{diag}(p) - p p^T) \, \sigma_s \in \mathbb{R}^{K \times K}$.
    \linebreak[2]
    For all $j \in \{0,1\}$, let us decompose
    $\partial_{a_j} \Omega_+(\mathcal{D}_s; \theta) \in \mathbb{R}^K$
    into its left and right terms
    $$L^\Omega(j,s) := (\partial_{a_j} r_s(a^0) \cdot H(\theta) ) \cdot \nabla \phi(r_s(a^0) \cdot H(\theta)) \in \mathbb{R}^K$$
    $$R^\Omega(j,s) := (r_s(a^0) \cdot \partial_{a_j} H(\theta) )  \cdot \nabla \phi(r_s(a^0) \cdot H(\theta)) \in \mathbb{R}^K$$

    Note that $r_s(a^0) \cdot a^0_s = \sum_{i \in [k]} p_i a^0_{s,i} - a_{s,y}^0 = (\sum_{i \in [k] \setminus \{y\}} p_i ) \gamma / \sigma_s = (1 - p_y) \gamma / \sigma_s$.
    Moreover since $p_y < 1$ because $p$ is in the image of $\operatorname{softargmax}$, and $\gamma < 0$, it holds $r_s(a^0) \cdot a^0_s < 0$.

    Furthermore, since $\phi$ is convex, it holds for any $z \in \mathbb{R}^h$ that
    $\phi(0) \geq \phi(z) + \nabla \phi(z) \cdot(0 - z)$
    thus $- \nabla \phi(z) \cdot z \leq - (\phi(z) - \phi(0))$.
    Hence, using $z := r_s(a^0) \cdot H(\theta) =\sum_{i \in [K]} p_i H_i(\theta) - H_y(\theta)$,
    \begin{equation}\label{eq:lss-phi-gap}\tag{E2}\begin{aligned}
    \left[ L^\Omega(s,s) \right]_y &=
    \Bigl( p_y H_y(\theta) - p_y \sum_{i \in [K]} p_i H_i(\theta)\Bigr) \,\sigma_s \cdot
    \nabla \phi\Bigl( \sum_{i \in [K]} p_i H_i(\theta) - H_y(\theta) \Bigr)
    \\ &\leq - \sigma_s p_y (\phi(z) - \phi(0))
    \end{aligned}\end{equation}

    Let us show additionally that $\phi(z) - \phi(0) \geq 0$.
    For that matter, observe that $A(\mathcal{D}_s; (u, 0)) = \lVert r_s(0) \cdot (0 \otimes w_s) \rVert_* = \lVert 0 \rVert_* =0$ thus $(u,0) \in \argmin A(\mathcal{D}_s; \cdot)$.
    By \Proposition{aligned-optima}, this implies that $(u,0) \in \argmin \Omega_+(\mathcal{D}_s; \cdot)$. However $\Omega_+(\mathcal{D}_s; (u,0)) = \phi(c_y) \geq \phi(0)$
    by definition of $y \in [K]$.
    Thus $\phi(0) \leq \argmin \Omega_+(\mathcal{D}_s; \cdot) \leq \Omega_+(\mathcal{D}_s; \theta) = \phi(z)$.

    On the other hand for $j = s$, the expression of the gradient $\partial_{a_s} A(\mathcal{D}_s; \theta)$ can be simplified using a property of homogeneous functions:
    for any $\zeta \in \mathbb{R}^K$, it holds $\zeta \cdot \nabla \psi(\zeta) = \psi(\zeta)$
    (see \Lemma{homogeneous-trick}),
    which we can use with $\zeta = (r_s(a^0) \cdot a_s^0) \,w_s \in \mathbb{R}^d$.
    Thus

    \begin{equation}\label{eq:xi-def}\tag{E3}
    \xi := \frac{\partial A}{\partial a_s}(\mathcal{D}_s; (u,a^0))
    = \frac{\psi(r_s(a^0) \cdot (a_s^0 \otimes w_s))}{r_s(a^0) \cdot a_s^0} \Bigl( \partial_{a_s} r_s(a^0) \cdot a_s^0 + r_s(a^0)  \Bigr)
    \end{equation}

    Let us check that $\xi_y > 0$.
    By the previous computation, $r_s(a^0) \cdot a_s^0 < 0$
    and thus $\psi((r_s(a^0) \cdot a_s^0 ) \, w_s) = \lvert r_s(a^0) \cdot a_s^0\lvert \, \psi(w_s) > 0$
    by homogeneity of $\psi = \lVert \cdot \rVert_*$ and the fact that $w_s \neq 0$.
    It remains to compute
    $$ \begin{aligned}
    \left[ \partial_{a_s} r_s(a^0) \cdot a_s^0 + r_s(a^0) \right]_y
    &= \sigma_s p_y (a_{s,y}^0 - \sum_{i \in [K]} p_i \,a_{s,i}^0) + p_y - 1
    \\ &= - \gamma p_y ( 1 - p_y ) + p_y - 1
    = - (1 - p_y) (1 + \gamma \, p_y)
    < 0
    \end{aligned}$$

    where the second equality uses the definition of $a_{s,i}^0 = \gamma \, v_i / \sigma_s$
    and $v_i = \mathds{1}_{i \neq y}$, and the fact that $\sum_{i \in [K], i \neq y} p_i = 1 - p_y$.
    The last inequality uses $(1 + \gamma \, p_y) > 0$, obtained by computation using the definition $p_y = \exp(\gamma v_y) / \sum_{i \in [K]} \exp(\gamma v_i)
    = 1 / (1 + (K-1) \exp(\gamma)) = 1 / (1 + (K-1) / 2)$
    and $K \geq 2$, to ensure $p_y \leq 2/3$ thus
    $1 + \gamma p_y \geq 1 + \log(1/2) \cdot 2/3 \approx 0.5379 > 0$.

    Thus by considering signs of each factor in \eqref{eq:xi-def}, we get $\xi_y > 0$, and also $\xi \neq 0 \in \mathbb{R}^K$.

    By definition $ R^\Omega(s,s) + L^\Omega(s,s) = \partial_{a_s} \Omega_+(\mathcal{D}_s; \theta)$
    and moreover by the equivalence assumption,
    $ \partial_\theta \Omega_+(\mathcal{D}_s; \theta)= \lambda_s \partial_\theta A(\mathcal{D}_s; \theta)$
    thus $R^\Omega(s,s) + L^\Omega(s,s) = \lambda_s \partial_{a_s} A(\mathcal{D}_s; \theta) = \lambda_s \xi$. Therefore
    $$ \left[ R^\Omega(s,s) \right]_y = \lambda_s \,
    \xi_y % \left[ \frac{\partial A}{\partial a_s}(\mathcal{D}_s; (u,a)) \right]_y
    - \left[ L^\Omega(s,s) \right]_y
    \underset{(S1)}{>} - \left[ L^\Omega(s,s) \right]_y
    \underset{(S2)}{\geq}  \sigma_s p_y \bigl( \phi(z) - \phi(0) \bigr)
    \underset{(S3)}{\geq} 0 $$
    By the previous calculation of $\xi_y > 0$ for ($S1$), the bound in \eqref{eq:lss-phi-gap} for ($S2$)
    and the fact $\phi(z) \geq \phi(0)$ shown previously ($S3$).
    In particular, this implies that $[R^\Omega(s,s)]_y >0$ therefore $R^\Omega(s,s) \neq 0 \in \mathbb{R}^K$.

    We are now almost ready to compute the squared cosine to show that the equation is violated,
    by observing that our choice of $a^0$ in $\theta = (u, a^0)$ ensures that
    $\forall (j,s) \in \{0,1\}^2, \> R^\Omega(j,s) = R^\Omega(j,j)$.
    Indeed, by definition of $r_s$ and $a^0$, for all $s \in \{0,1\}$ it holds
    ${ r_s(a^0) = \partial_{u} \ell(y, a_s^0 \sigma_s) = \partial_u \ell(y, \gamma \, v) }$.
    Hence $r_s(a^0) = r_j(a^0)$ for all $(j,s) \in \{0,1\}^2$ even if $j \neq s$.
    Thus exchanging them in the definition
    ${ R^\Omega(j,s) = (r_s(a^0) \cdot \partial_{a_j} H(\theta)) \cdot \nabla \phi( r_s(a^0) \cdot H(\theta)) }$, we get that ${R^\Omega(j,s) = R^\Omega(j,j) }$.
    As shown before, for any $j \in \{0,1\}$ the diagonal term
    $R^\Omega(j,j)$ is non-null, and this property permits transferring it off-diagonal.
    Hence if $j \neq s$ we get
    ${\partial_{a_j} \Omega_+(\mathcal{D}_s; \theta) = L^\Omega(j,s) + R^\Omega(j,s) = R^\Omega(j,s) \neq 0}$
    because $R^\Omega(j,s) = R^\Omega(j,j) \neq 0$ and $L^\Omega(j,s) = 0$ as observed before.

    By using the PSD form
    $Q((u,a), (u', a')) := \sum_{j \in \{0,1\}} \sum_{i \in [k]} a_{j,i} a'_{j,i}$,
    we get by definition
    ${ \pnorm{Q}{\partial_\theta \Omega_+(\mathcal{D}_s; \theta)}^2
    % = \lVert \partial_{a_s} \Omega_+(\mathcal{D}_s; \theta) \rVert_2^2 + \lVert \partial_{a_j} \Omega_+(\mathcal{D}_s; \theta) \rVert_2^2 } $
    = \sum_{j \in \{0,1\}} \lVert \partial_{a_j} \Omega_+(\mathcal{D}_s; \theta) \rVert_2^2 } $
    and we can cancel off-diagonal ($j\neq s$) terms $\partial_{a_j} A(\mathcal{D}_s; \theta) = 0$
    to get
    $ { \pnorm{Q}{ \partial_\theta A(\mathcal{D}_s; \theta) }^2
    = \lVert \partial_{a_s} A(\mathcal{D}_s; \theta) \rVert_2^2}$.
    Using the Cauchy-Schwarz inequality,
    $$\begin{aligned} \langle \partial_\theta \Omega_+(\mathcal{D}_0; \theta), \partial_\theta A(\mathcal{D}_0; \theta) \rangle_Q
        &= \partial_{a_0} \Omega_+(\mathcal{D}_0; \theta) \cdot \partial_{a_0} A(\mathcal{D}_0; \theta)
        + \partial_{a_{1}} \Omega_+(\mathcal{D}_0; \theta) \cdot \partial_{a_{1}} A(\mathcal{D}_0; \theta)
        \\ &= \partial_{a_0} \Omega_+(\mathcal{D}_0; \theta) \cdot \partial_{a_0} A(\mathcal{D}_0; \theta) + 0
         \\ &\leq \lVert \partial_{a_0} \Omega_+(\mathcal{D}_0; \theta) \rVert_2 \cdot \lVert \partial_{a_0} A(\mathcal{D}_0; \theta) \rVert_2
    \end{aligned}$$
    Thus, using this inequality on the numerator and canceling the factors due to
    $\partial_{\theta} A(\mathcal{D}_0; \theta) = \xi \neq 0$,
    $$\begin{aligned}
        \frac{ \langle \partial_\theta \Omega_+(\mathcal{D}_0; \theta), \partial_\theta A(\mathcal{D}_0; \theta) \rangle_Q^2}
        {\lVert \partial_\theta \Omega_+(\mathcal{D}_0; \theta) \rVert_Q^2 \, \lVert \partial_\theta A(\mathcal{D}_0; \theta) \rVert_Q^2}
        \leq \frac{ \lVert \partial_{a_0} \Omega_+(\mathcal{D}_0; \theta) \rVert_2^2 }{ \lVert \partial_{a_0} \Omega_+(\mathcal{D}_0; \theta) \rVert_2^2 + \lVert \partial_{a_{1}} \Omega_+(\mathcal{D}_0; \theta) \rVert_2^2 }
        < 1
    \end{aligned}$$
    where the inequality follows from the fact that $\partial_{a_1} \Omega_+(\mathcal{D}_0; \theta) \neq 0$ as shown above, thus the alignment equality of \Proposition{cosine} is violated at $\theta$, and the proof by contradiction is concluded.
\end{proof}

\pagebreak
\section{Experimental setup and details}\label{appendix:experimental-details}

\subsection{AT-PGD sampling procedure}\label{appendix:experiment-at-pgd}

Samples are drawn from the CIFAR-10 dataset ({\small \texttt{torchvision.datasets.CIFAR10}}) and augmented using random crops and horizonal flips (respectively {\small \texttt{v2.RandomCrop(32, padding=4)}} and {\small \texttt{v2.RandomHorizontalFlip(p=0.5)}} using {\small \texttt{torch.transforms.v2}}).

After sampling $(x, y) \in \mathbb{R}^d \times [K]$, a 10-step gradient ascent
is applied to generate an adversarial sample for
$\mathcal{L}(\theta, \cdot, y) : \mathbb{R}^d \to \mathbb{R}$, as $x_{k+1} = P(x_k + \eta_0 \, g_k)$
where ${g_k = \operatorname{sgn}(\partial_x \mathcal{L}(\theta,x_k,y)) \in \{\pm 1\}^{d}}$
and $P$ is the projection to the convex set $\{ u \in \interval{0}{1}^d \mid \lVert u - x \rVert_\infty \leq \delta \}$.
We use initial uniform random noise, that is to say $x_0 = P(x + \delta \cdot \varepsilon)$ where $\varepsilon \in \mathbb{R}^d$ has independent identically distributed coordinates with $\varepsilon_i \sim \mathcal{U}(-1, +1)$.
The numerical values used are $\delta = 8 / 255 \approx 0.3137$ and stepsize $\eta_0 = 0.2 \times \delta$.
This ascent procedure is performed in parallel on batches of 64 samples.
This is strictly speaking closer to a multi-step FGSM (Fast Gradient Sign Method, see \citet{goodfellow2015explaining})
procedure, but has been referred to as AT-PGD since at least \citet[Section 2.1]{madry2018towards}, so we keep the historical naming convention.

\subsection{Two-layer networks on CIFAR-10 classification}\label{appendix:experiment-mlp}

For \Figure{projected-2lp}, we use a two-layer architecture with biases, explicitly
$f : \Theta \times \mathbb{R}^d \to \mathbb{R}^K$ with parameter space
$\Theta = \mathbb{R}^{m \times d} \times \mathbb{R}^m \times \mathbb{R}^{m \times K} \times \mathbb{R}^K$
and $\sigma = \operatorname{ReLU}$ ({\small \texttt{torch.nn.functional.relu}})
$$ f : ( (w,b, a, c), x) \mapsto c + \frac{1}{\sqrt{m}} \sum_{i \in [m]} a_i \, \sigma(w_i \cdot x + b_i) $$
Parameters $a \in \mathbb{R}^{m \times K}$ and $c \in \mathbb{R}^K$ are initialized as zero, and $w_{i,j} \sim \mathcal{N}(0,1)$ and $b_i \sim \mathcal{N}(0,1)$ for all $i \in [m]$ and $j \in [d]$.
The network is trained with AT-PGD samples on CIFAR-10 and SGD optimizer ({\small \texttt{torch.optim.SGD}}) with step size $\eta = 10^{-1}$ and momentum $\mu = 0.9$.
When freezing the first neuron to observe parallel slices, we multiply the stepsize applied to $a_0 \in \mathbb{R}^K$ by zero, and check numerically that it remains constant for the second part of training.
The first part (neuron not frozen) is run for 10 epochs with batch size 64, and the second part (frozen neuron) is run for an additional 40 epochs, with the same batch size using AT-PGD samples.
Optimizer parameters are zeroed-out after the first ten epochs, to ensure that no momentum persists for the second (frozen neuron) part.

\subsection{WideResNet-34-10 on CIFAR-10 classification}\label{appendix:experiment-wrn}

For Wide ResNets, we use the implementation of \href{https://github.com/RobustBench/robustbench}{https://github.com/RobustBench/robustbench}, with depth 34 and widening factor of 10, with the default initialization of the implementation provided.
We train networks with AT-PGD samples and SGD optimizer
({\small \texttt{torch.optim.SGD}})
with a cosine scheduler starting with step size $\eta_0 = 10^{-1}$ and ending at $\eta=10^{-3}$ after 100 epochs,
with a momentum parameter of $\mu = 0.9$.
When freezing the first neuron of the last layer ({\small \texttt{model.fc.weight}} with shape $10 \times 640$), we also reset
momentum optimizer parameters to ensure that no momentum persists across the freeze,
we multiply the learning rate of the first neuron by zero, and we check numerically that the parameters of this neuron remain constant between checkpoints.

On a single A100 GPU,
training with this procedure takes between four and five minutes per epoch, for a total
(20 epochs for \Fig{wideresnet-E10} and \Fig{wideresnet-E20}
and 100 for \Fig{wrn-cosine}) of nearly nine hours.
For \Figure{wideresnet-E10} and \Figure{wideresnet-E20}, each pixel requires a full pass on the training data, computing a single derivative with respect to the input,
taking approximately thirty to forty-five seconds. The total with two scans of $51 \times 51$ pixels adds up to about sixty hours.
Including training, validation, and earlier runs for these visualizations at different
resolution / scale and previous identical experiments on smaller ResNet architectures
brings the total of A100-hours used to the order of three hundred hours.

\end{document}